\documentclass{article}
\usepackage{iclr2027_conference,times}

\usepackage{amsmath,amsfonts,bm}

\def\eqref#1{equation~\ref{#1}}

\def\1{\bm{1}}

\DeclareMathAlphabet{\mathsfit}{\encodingdefault}{\sfdefault}{m}{sl}
\SetMathAlphabet{\mathsfit}{bold}{\encodingdefault}{\sfdefault}{bx}{n}

\usepackage{hyperref}
\usepackage{url}
\usepackage{graphicx}
\usepackage{comment}
\usepackage{booktabs}
\usepackage{multirow}
\usepackage{rotating}
\usepackage{graphicx}
\usepackage{kotex}
\usepackage[table]{xcolor}
\definecolor{iclrlink}{rgb}{0.10,0.20,0.55}
\hypersetup{colorlinks=true, linkcolor=iclrlink, citecolor=iclrlink, urlcolor=iclrlink}
\usepackage{algorithm}
\usepackage{algpseudocode}

\title{VolS-GS: Relightable Gaussian Splatting \\ with Volumetric Subsurface Scattering}

\author{Junyeong Ahn, \; Jaegul Choo\thanks{Corresponding author.} \\
KAIST AI \\
\texttt{\{justinahn, jchoo\}@kaist.ac.kr}
}

\usepackage{amsthm}
\usepackage{capt-of}
\usepackage{wrapfig}
\usepackage{placeins}
\AtBeginDocument{\abovedisplayshortskip=\abovedisplayskip \belowdisplayshortskip=\belowdisplayskip}
\newtheorem{lemma}{Lemma}
\newtheorem{proposition}{Proposition}

\newtheorem{remark}{Remark}

\iclrfinalcopy
\renewcommand{\headrulewidth}{0pt}
\begin{document}

\setcounter{footnote}{1}
\maketitle

\vspace{-0.20in}
\centerline{\url{justin4ai.github.io/VolS-GS}}

\begin{center}
\vspace{-0.15\baselineskip}
\includegraphics[width=\linewidth]{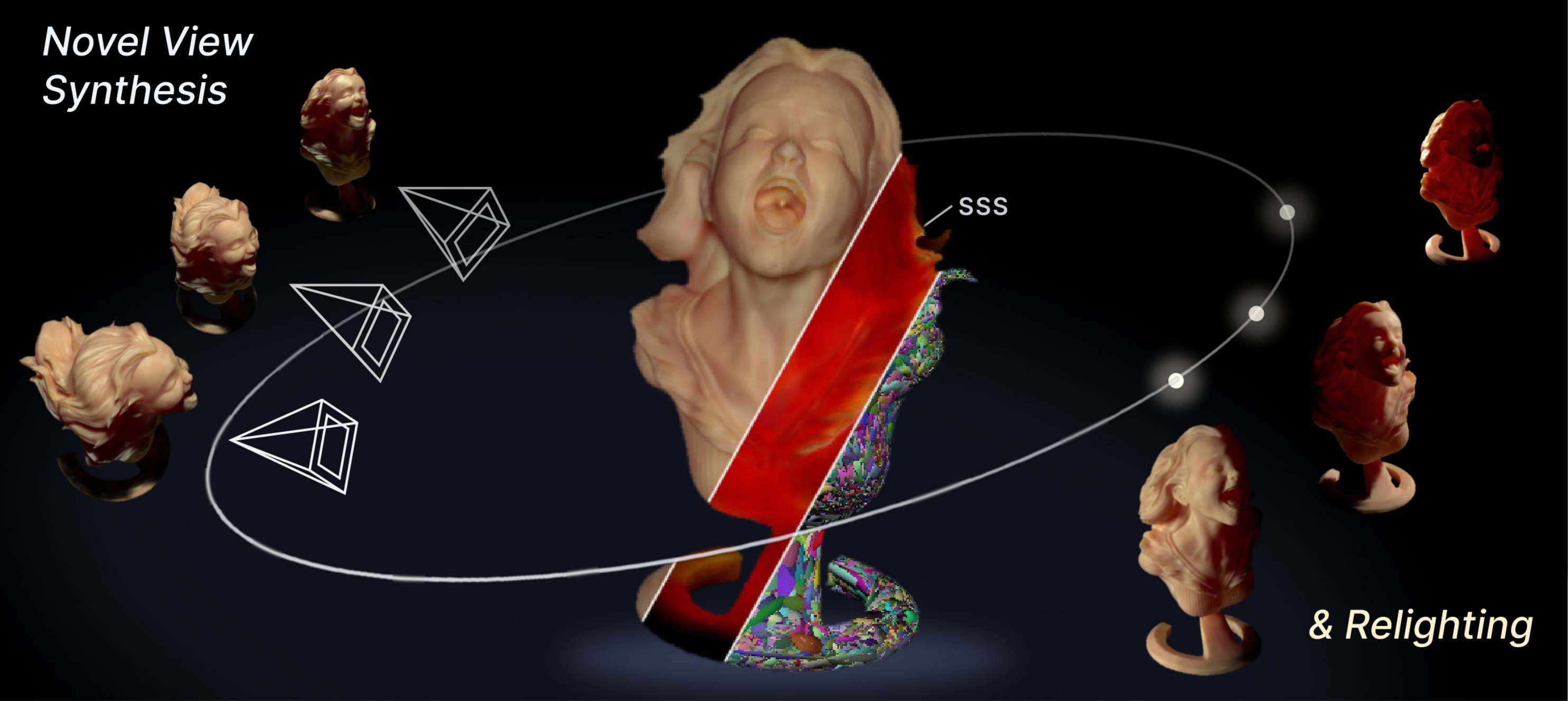}
\captionof{figure}{\textbf{VolS-GS} proposes a high-quality relighting pipeline with three different modalities. \textbf{Center:} a rendered image can be decomposed into normal, diffuse, specular, subsurface scattering (SSS) and shadow. \textbf{Left:} VolS-GS faithfully synthesizes novel view renders. \textbf{Right:} relighting.}
\label{fig:teaser}
\vspace{2mm}
\end{center}

\begin{abstract}
We present VolS-GS, a relightable Gaussian splatting framework that reconstructs objects from
one-light-at-a-time (OLAT) captures and renders them under novel lighting and viewpoints.
Relightable Gaussian Splatting methods typically model appearance independently at each
primitive, which makes non-local effects difficult to represent. This limitation is particularly
apparent for subsurface scattering, where light entering the object at one location can emerge
at another. Rather than modeling this effect solely with a neural network or a local kernel at
each primitive, we use the spatial support of the Gaussian scene as the domain of a
differentiable finite-volume transport solver, so that light can propagate through the object's
interior. A small network predicts scattering and absorption coefficients for each Gaussian, and
the solve redistributes incident light through the resulting field. The coefficients are fit to
images rather than measured, so the solve supplies a transport-shaped path for aggregating
per-primitive appearance, not a measurement of the material. To keep the learned shadow and specular
terms from taking over the other components, our shadow term is predicted from visibility
together with the transmittances and the scattering the solve produces, and a regularizer
suppresses specular highlights in regions the shadow term predicts to be unlit. Experiments on three OLAT
benchmarks show that VolS-GS consistently improves relighting quality on held-out lights and
views.

\end{abstract}

\section{Introduction}
\label{sec:intro}

Relighting a captured object requires predicting its appearance under lighting conditions that were not observed during capture. One-light-at-a-time (OLAT) capture provides a controlled setting for this task by recording each image under a known point light, and the model must render the object from new viewpoints and under new light positions.
Neural radiance fields made relighting practical by learning appearance as a function of position and illumination~\citep{zeng2023nrhints,jin2023tensoir}, and 3D Gaussian splatting~\citep{kerbl20233dgs} replaced the implicit field with an explicit set of spatial primitives that renders far faster. Existing Gaussian approaches~\citep{bi2024gs3,dihlmann2024sssgs,zheng2026ssdgs} give every primitive its own
appearance model, evaluated from that primitive's own attributes.

Non-local transport is not beyond the representation. Radiosity-GS~\citep{jiang2025radiositygs}
transports light between primitives with a differentiable radiosity solve, and Ref-Gaussian~\citep{yao2025refgaussian} traces
inter-reflections among them. Both route light \emph{around} the object, through the space outside it. Subsurface scattering routes
light \emph{through} it. The light enters at one surface point, travels inside the medium, and leaves at another. No Gaussian method
solves for that interior path. SSS-GS~\citep{dihlmann2024sssgs} conditions a scattering term on the light direction and
SSD-GS~\citep{zheng2026ssdgs} evaluates a closed-form diffusion kernel at each primitive, so in both the subsurface response remains a
function of one primitive's own attributes. A BSSRDF couples light entering the object to radiance emerging elsewhere, which introduces a surface integral and makes direct evaluation expensive. Classical diffusion models reduce this coupling to a local kernel under assumptions such as a homogeneous medium and a simple boundary~\citep{jensen2001bssrdf, donner2005light, habel2013photon}. A per-primitive evaluation of such a kernel therefore has no interior path (Figure~\ref{fig:kernel}).

\begin{figure}[!htb]
    \centering
    \includegraphics[width=\linewidth]{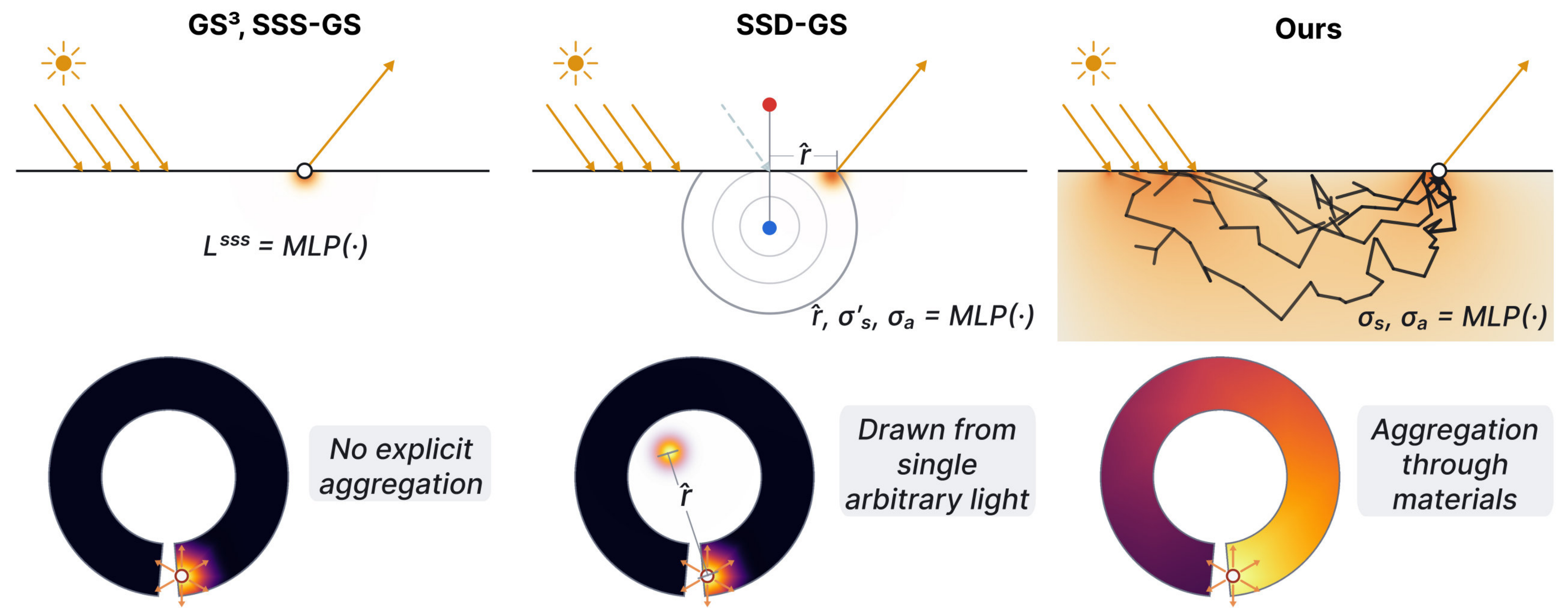}
    \caption{\textbf{Per-primitive kernels versus transport through the medium.} Top: where each
    method computes its subsurface response. Bottom: a C-shaped medium lit at the opening. The dipole
    depends on Euclidean distance, so its support crosses the gap, while our solve aggregates the subsurface
    response through the learned medium.}
    \label{fig:kernel}
\end{figure}

The Gaussian representation already provides the spatial support on which subsurface scattering can be computed. We turn that into an occupancy field on a coarse grid and solve on the occupied cells. Each Gaussian then reads the solved field back at its own
position. The grid carries no parameters of its own, so the domain is the Gaussians themselves rather than an extracted mesh. VolS-GS learns scattering and absorption properties over this domain and uses them in a finite-volume transport solve. The solve produces the multiply scattered field, while a separate single-scattering term adds the once-scattered light integrated along the view chord, which the angle-averaged fluence cannot represent. An adjoint formulation makes the solver differentiable without unrolling its iterations, so the medium is learned end to end from images alone. Shadow modeling shares transmittance cues with the transport solve, and a regularizer discourages the specular branch from absorbing transported light.

In summary, we present VolS-GS, a relightable Gaussian splatting method, with the following contributions. (1) A differentiable volumetric subsurface scattering defined on the Gaussian scene's occupancy, allowing light to propagate across primitives. (2) A shadow term that reads the same transport field and a regularizer on the specular branch, which together keep the learned shadow and the learned specular from trading against each other.
A full discussion of related work is deferred to Appendix~\ref{sec:related}.

\section{Method}
\label{sec:method}

\begin{figure}[!t]
    \centering
    \includegraphics[width=\linewidth]{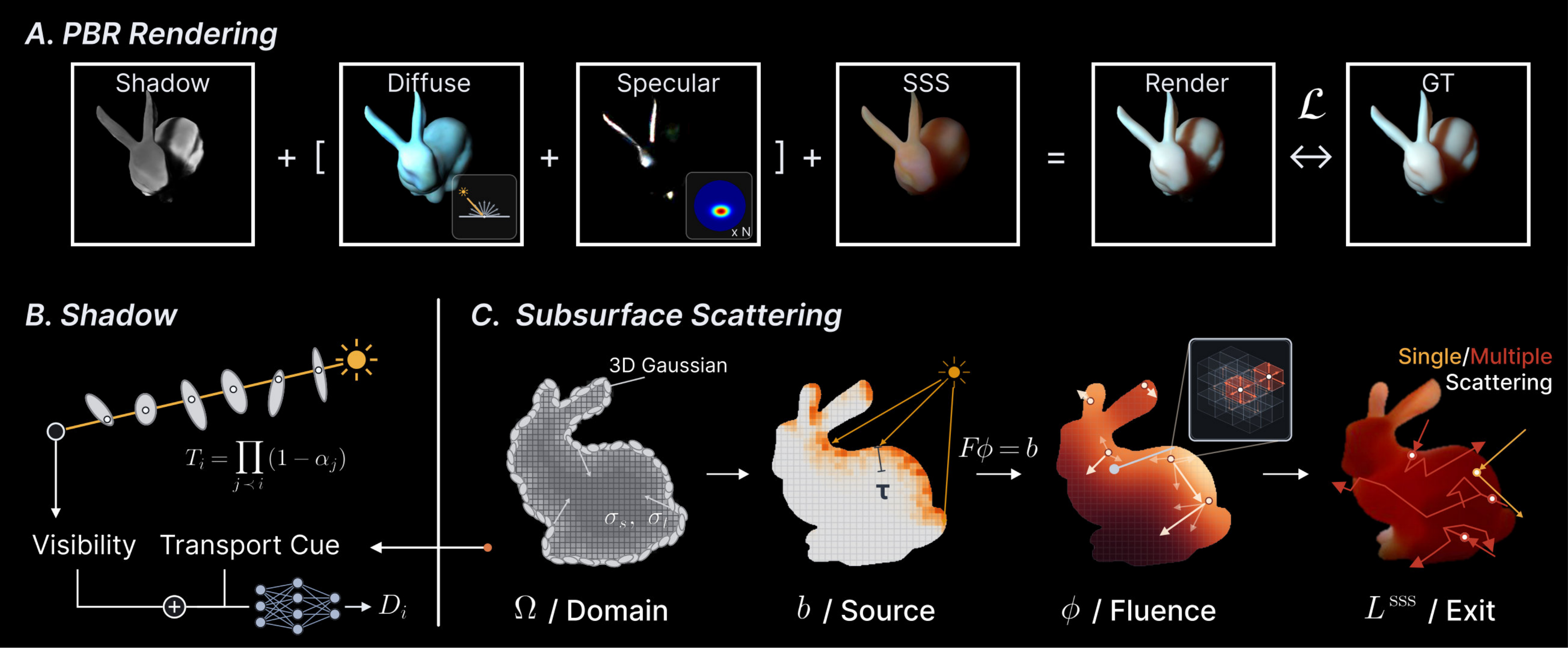}
    \caption{\textbf{Overview of VolS-GS.} \textbf{(A)} A render is the sum of a shadowed
    surface response and a subsurface term, compared against the captured frame.
    \textbf{(B)} The shadow is predicted from two cues gathered along the ray to the light,
    the fraction the Gaussians block and the fraction the medium transmits.
    \textbf{(C)} The subsurface term is propagated rather than predicted per primitive. The Gaussians define a
    domain $\Omega$, the light entering it forms a source $b$, one linear solve $F\phi=b$
    spreads that source through the medium, and each Gaussian reads the result back as
    outgoing radiance, so light can enter in one place and leave in another.}
    \label{fig:method}
\end{figure}

We use the Gaussian cloud itself as the spatial support of a transport computation, solving for one unknown field, the fluence, on the cells the Gaussians occupy. Its source is the light that has scattered once inside the medium, and its solution is read back at each primitive as outgoing radiance (Figure~\ref{fig:method}).

\subsection{Preliminaries}
\label{sec:prelim}

3D Gaussian splatting~\citep{kerbl20233dgs} represents a scene as a set of anisotropic Gaussians $\{g_i\}_{i=1}^{N}$, each with position $\mathbf{x}_i$, rotation and scale $(\mathbf{q}_i, \mathbf{s}_i)$ defining its covariance, opacity $o_i$, and a color $c_i$. Rendering projects the Gaussians to screen space and alpha-composites them front to back,
\begin{equation}
C(\mathbf{p}) \;=\; \sum_{i \in \mathcal{S}(\mathbf{p})} c_i\, \alpha_i \prod_{j<i} (1-\alpha_j),
\label{eq:composite}
\end{equation}
where $\alpha_i$ combines $o_i$ with the projected footprint and $\mathcal{S}(\mathbf{p})$ is the depth-sorted set of Gaussians covering pixel $\mathbf{p}$. Our method changes only the computation of the per-Gaussian color $c_i$, leaving the rasterizer untouched.

\paragraph{Diffusion-based light transport.}
Where scattering dominates absorption, light loses its directionality after a few events and multiple scattering
is well approximated by diffusion of the fluence $\phi$, the direction-averaged light
density --- the zeroth angular moment of the radiative transfer equation, kept to first order
(P1)~\citep{ishimaru1978wave,stam1995multiple,jensen2001bssrdf}. The approximation is weakest at the
boundary and for the first, still-directional event, which we therefore keep separate
(Section~\ref{sec:sss2}).

\subsection{Relighting formulation}
\label{sec:formulation}

Under a point light with incident direction $\boldsymbol{\omega}_l$, view direction $\boldsymbol{\omega}_o$, and distance $d_i$ to Gaussian $g_i$, the outgoing radiance of the primitive is the sum of shadowed direct reflection and subsurface scattering,
\begin{equation}
c_i \;=\; L^{\mathrm{d}}_i + L^{\mathrm{s}}_i + L^{\mathrm{sss}}_i,
\qquad
L^{\mathrm{d}}_i \;=\; \frac{D_i}{d_i^{2}}\; \mathbf{a}_i\,\langle \mathbf{n}_i, \boldsymbol{\omega}_l\rangle,
\qquad
L^{\mathrm{s}}_i \;=\; \frac{D_i}{d_i^{2}}\; f^{\mathrm{s}}_i(\boldsymbol{\omega}_l, \boldsymbol{\omega}_o),
\label{eq:skeleton}
\end{equation}
where $1/d_i^{2}$ is the point light's inverse-square falloff and the learned decay $D_i \in [0,1]$ is the additional attenuation from shadowing. The diffuse term is Lambertian with per-Gaussian albedo $\mathbf{a}_i$ and a learned shading normal $\mathbf{n}_i$, and the specular lobe $f^{\mathrm{s}}_i$ is a mixture of anisotropic spherical Gaussians~\citep{xu2013asg} whose amplitudes are conditioned on the light direction, an expressive basis for stretched and anisotropic highlights~\citep{bi2024gs3} in place of an explicit microfacet BRDF~\citep{walter2007microfacet}. The decay multiplies the two direct terms only. Light that leaves through the subsurface channel has already been attenuated inside the solve, by the march that builds its source and by absorption along the path, so applying $D_i$ to it as well would attenuate the same travel twice; Section~\ref{sec:ablation} measures the alternative. $D_i$ is predicted from a geometric visibility $V_i$ splatted from the light's viewpoint, together with transport cues (Section~\ref{sec:erase}), and the colored Gaussians are composited by Eq.~\ref{eq:composite}. Section~\ref{sec:sss2} defines $L^{\mathrm{sss}}_i$.

\subsection{Gaussian-Supported Transport Domain}
\label{sec:sss1}

\paragraph{Transport Domain.}

We lay a $64^{3}$ grid over an axis-aligned box around the cloud and rasterize the spatial support
of the Gaussians into a fractional occupancy field $O(\mathbf{x}) \in [0,1]$ on it, taking the
occupied cells as the transport domain $\Omega$. The
occupancy is kept fractional inside the domain rather than rounded to a binary mask, because the
boundary term of the finite-volume operator introduced below scales with interface area, which a binary mask
overcounts~\citep{johansen1998cartesian}. A low occupancy cutoff decides which cells enter $\Omega$
(Appendix~\ref{app:impl}). Gaussians fit the surface, so rasterizing them leaves a hollow shell, and a shell has no interior for
light to cross. The interior is therefore filled by a
generalized winding number over a subsample of them~\citep{barill2018winding},
which separates inside from outside without a watertight mesh. The per-Gaussian
coefficients are then splatted onto the same grid with a kernel spanning several
cells. Each interior cell takes the coefficients of the nearest
surface cell, propagated inward on a coarser grid, the construction used for extension velocities in level-set methods~\citep{adalsteinsson1999extension} (Appendix~\ref{app:impl}).

Gradients flow back through this splatting, and the domain itself adds no learnable parameters: it is
read off the Gaussians at every rebuild. Following the classical two-part decomposition, the multiple-scattering
solve is driven by the singly scattered light~\citep{stam1995multiple,arbree2011heterogeneous}. For each active
cell $\mathbf{c}$ we march toward the light through $O$ and accumulate $\tau(\mathbf{c}) = \int O\,ds$, the path
length weighted by occupancy rather than the geometric distance, so that empty space between
disconnected parts of the object contributes nothing to the extinction, and set
\begin{equation}
E(\mathbf{c}) = \frac{e^{-\sigma_t(\mathbf{c})\, \tau(\mathbf{c})}}{d(\mathbf{c})^{2}}, \qquad
b(\mathbf{c}) = \sigma_s(\mathbf{c})\, E(\mathbf{c})\, O(\mathbf{c}).
\label{eq:source}
\end{equation}

Here $E$ is the light that survives the march to the cell, and $b$ is the density of first-scattering events there, which acts as the source term of the solve.

\paragraph{Medium Coefficients.}

A small MLP maps each Gaussian's encoded position, its per-primitive latent, and the encoded light and view
directions to per-channel scattering and absorption coefficients,
\begin{equation}
\big(\boldsymbol{\sigma}_s(g_i),\, \boldsymbol{\sigma}_a(g_i)\big) \;=\; \mathrm{MLP}_{\mathrm{mat}}\big(\gamma(\mathbf{x}_i),\, \mathbf{z}_i,\, \gamma(\boldsymbol{\omega}_l),\, \gamma(\boldsymbol{\omega}_o)\big) \;\in\; \mathbb{R}^3\times\mathbb{R}^3,
\label{eq:medium}
\end{equation}
where $g_i$ is the $i$-th Gaussian, $\gamma$ a frequency encoding and $\mathbf{z}_i$ its per-Gaussian latent
(Appendix~\ref{app:impl}). Extinction is $\sigma_t = \sigma_s + \sigma_a$ and the diffusion coefficient it
implies is $\kappa = 1/(3\sigma_t)$, both per channel. We refer to $\sigma_s$, $\sigma_a$ and $\sigma_t$ collectively as the medium coefficients.

\subsection{Volumetric subsurface scattering}
\label{sec:sss2}

\paragraph{Source and Diffusion Solve.}
We solve for the fluence $\phi$, one scalar per cell and per color channel on $\Omega$, which obeys
\begin{equation}
-\nabla\!\cdot\!\big(\kappa(\mathbf{x})\, \nabla\phi\big) + \sigma_a(\mathbf{x})\,\phi = b(\mathbf{x})
\quad \text{on } \Omega, \qquad
\kappa\,\partial_n \phi + \tfrac{1}{2A(\eta)}\,\phi = 0 \quad \text{on } \partial\Omega,
\label{eq:pde}
\end{equation}
The second equation is a Robin condition stating that a fraction of the light reaching the
surface leaves and the rest is internally reflected, where $A(\eta)$ is the classical internal-reflection factor, equal to one at $\eta = 1$~\citep{jensen2001bssrdf,arbree2011heterogeneous}. Integrating Eq.~\ref{eq:pde} over one cell and splitting its faces into those shared with another
active cell and those that meet air gives the standard finite-volume
stencil~\citep{patankar1980numerical}, per cell,
\begin{equation}
\underbrace{\frac{1}{h^{2}}\!\!\!\sum_{d\in\mathcal{N}(c)}\!\!\!\kappa_{cd}\,\alpha_{cd}\,(\phi_c-\phi_d)}_{D}
\;+\;
\underbrace{\sigma_a(c)\,O(c)\,\phi_c}_{M}
\;+\;
\underbrace{\frac{|\nabla O|_c}{2A(\eta)\,h}\,\phi_c}_{S}
\;=\; b(c),
\label{eq:fv}
\end{equation}
with $h$ the cell size, $\phi_c$ the fluence in cell $c$, $\mathcal{N}(c)$ its six neighbors,
$\kappa_{cd} = 2\kappa_c\kappa_d/(\kappa_c+\kappa_d)$ the finite-volume face conductance of two half
cells and $\alpha_{cd} = \min(O_c, O_d)$ the open fraction of that face. The factor $|\nabla O|_c$ in the third block is the interface area the cell presents to the outside, which is why a fractional occupancy matters: a binary mask puts the whole face area on every boundary cell. $D$ transports fluence between neighboring cells and is driven by the difference $\phi_c - \phi_d$, $M$
absorbs it, and $S$ lets it escape through the faces that meet air. Stacking the rows gives $F\phi = b$ with $F = D + M + S$, where $D$ discretizes the divergence term of Eq.~\ref{eq:pde}, $M$ its absorption term, and $S$ its Robin boundary condition.

We solve this system with preconditioned conjugate gradients rather than forming $F^{-1}$, and we
obtain gradients without storing the iterations. The equation is linear in the source, so several lights superpose given the medium, which is how
Appendix~\ref{app:envmap} relights under an environment map without retraining. Differentiating $F\phi = b$ gives $F\,\partial\phi = \partial b - (\partial F)\,\phi$, so the
sensitivity of $\phi$ is itself a solve against the same operator. Because $F$ is symmetric no
transpose is needed, and solving $F\lambda = \partial\mathcal{L}/\partial\phi$ once yields the
gradients with respect to $b$ and to the diagonal of $M + S$
(Propositions~\ref{prop:spd} and~\ref{prop:adjoint}). The backward pass therefore costs one extra
solve and its memory does not grow with the iteration count. Differentiating a transport solve
through its adjoint rather than through its iterations is the standard treatment in differentiable
rendering~\citep{nimierdavid2019mitsuba2,nimierdavid2020radiative,vicini2021pathreplay}. Here the operator
is assembled on the Gaussians' own occupancy. The face conductances are held at
stop-gradient, so the medium is learned through the terms it enters directly.

\paragraph{Exit Radiance.}
Multiply scattered radiance leaves the boundary as
\begin{equation}
L_{\mathrm{ms}}(g_i) = \frac{1 + 1/A(\eta)}{4\pi}\; F_t(\cos\theta_o)\; \phi\big(\mathbf{x}_i - \delta\,\mathbf{n}_i\big),
\label{eq:exit}
\end{equation}
with $F_t$ the Fresnel transmittance and the fluence sampled one cell inward ($\delta = h$), since the splats sit on the
voxelized boundary where interpolation would mix interior values with the zeros outside. We set the exit index to $\eta = 1$ on every scene, so $A(\eta)$ and $F_t$ reduce to identity and the boundary scale is absorbed by the calibration below.

Diffusion is angle-averaged and
does not represent the once-scattered light near the surface, which the classical decomposition
adds separately~\citep{hanrahan1993reflection,jensen2001bssrdf}. Along the view ray through $g_i$, in material
depth $u$,
\begin{equation}
L_{1}(g_i) = \frac{1}{4\pi} \int_{0}^{U_i} Q\big(\mathbf{x}(u)\big)\; e^{-\sigma_t u}\, du,
\qquad Q = \sigma_s\,E,
\label{eq:single}
\end{equation}
where $U_i$ is the material chord, the part of the view ray that lies inside the medium, marched on the
same occupancy field as $\tau$ so that the two path lengths are measured consistently. $Q$ is the same
source the solve is driven from, so the term costs one $32$-sample quadrature and no new state. The channel a Gaussian finally emits is
\begin{equation}
L^{\mathrm{sss}}_i \;=\; \mathbf{k}^{\mathrm{sss}}_i \odot g_{\mathrm{cal}} \left[ L_1(g_i) + L_{\mathrm{ms}}(g_i) \right],
\label{eq:sss}
\end{equation}
where $\mathbf{k}^{\mathrm{sss}}_i$ is a learned per-primitive, per-channel amplitude and $g_{\mathrm{cal}}$ one
learned scalar per scene, both calibrating solver units to rasterizer radiometry. Where the medium is
thin the solve's support collapses onto the primitive's own neighborhood and the channel carries
almost nothing, so the model falls back to local shading rather than colliding with it
(Remark~\ref{rem:degen}).

The coefficients are learned from images and conditioned on the light and view directions, so they are not
optical constants of the material.

\subsection{Shadow and Specular Disambiguation}
\label{sec:erase}

Both the specular basis and the shadow decay are learned, and a highlight that a shadow cancels leaves no trace in the rendered image, so the two can trade against each other freely. The specular component can therefore explain smooth transported illumination whenever the shadow term suppresses the direct component, which is exactly the appearance the solve of Section~\ref{sec:sss2} should own. 

Shadowing has two stages. We first estimate geometric visibility from the light by rendering the Gaussian scene with the rasterizer of Eq.~\ref{eq:composite}.

\begin{equation}
V_i \;=\; \frac{\sum_{\mathbf{p}} G_i(\mathbf{p})\, T_i(\mathbf{p})}{\sum_{\mathbf{p}} G_i(\mathbf{p})},
\qquad
T_i(\mathbf{p}) \;=\; \prod_{j \prec i} \big(1-\alpha_j(\mathbf{p})\big),
\label{eq:vis}
\end{equation}
where $\mathbf{p}$ ranges over the light-view pixels in the footprint $G_i$ of $g_i$ and $j \prec i$ are the Gaussians closer to the light. $V_i \in [0,1]$ is a transmittance, one when nothing lies between $g_i$ and the light and zero when it is fully occluded. A small network then predicts the attenuation $D_i \in [0,1]$ from this visibility together with transmittance cues taken from the transport field,
\begin{equation}
D_i \;=\; \sigma\Big(g_\theta\big(\gamma(\mathbf{x}_i),\, \gamma(\boldsymbol{\omega}_l),\, \mathbf{z}_i,\, \mathbf{h}_i\big)\Big),
\qquad
\mathbf{h}_i \;=\; \big[\, V_i,\;\; e^{-\bar\sigma_{t,i}\,\tau_i},\;\; e^{-\bar\sigma_{t,i}\,U_i},\;\; r_i,\;\; \chi_i \,\big],
\label{eq:decay}
\end{equation}
with $\mathbf{z}_i$ the same per-Gaussian latent as in Eq.~\ref{eq:medium} and the hint vector $\mathbf{h}_i$ gradient-stopped, for the same reason as the face conductances: the shadow branch reads the transport field, and letting it push gradients back would let it reshape the medium to make shadows cheaper, fitting $\sigma_t$ to the shadow residual rather than to transport. $\tau_i$ and $U_i$ are the material path lengths toward the light and along the view chord, marched on the same occupancy grid as the transport, and $\bar\sigma_{t,i}$ is the primitive's mean extinction. 

Both the shadow decay and the transport solve therefore depend on the same extinction field \(\sigma_t\). The last two entries compare the channel with the undecayed direct terms, $r_i = \log\big(1 + \bar{L}^{\mathrm{sss}}_i / \overline{L^{\mathrm{d}}_i + L^{\mathrm{s}}_i}\big)$ and $\chi_i = L^{\mathrm{sss}}_{i,R} / \bar{L}^{\mathrm{sss}}_i$, bars denoting the mean over color channels. Together they indicate the share of the local light that is transported rather than reflected.

We break the ambiguity with a shadow-aware prior on the rendered images,
\begin{equation}
\mathcal{L}_{\text{leak}} \;=\; \frac{1}{|\mathcal{P}|}\sum_{\mathbf{p}\in\mathcal{P}} \tilde{S}(\mathbf{p})\,\big(1 - \mathrm{sg}[\tilde{D}(\mathbf{p})]\big),
\label{eq:erase}
\end{equation}
where $\mathcal{P}$ is the set of pixels the object covers, $\tilde{S}$ and $\tilde{D}$ are the per-Gaussian specular radiance and decay after alpha compositing along the ray, and $\mathrm{sg}[\cdot]$ stops the gradient, so the penalty acts on the specular term only and cannot be avoided by brightening shadows.

\subsection{Training}
\label{sec:training}

All components are trained jointly from scratch. The objective is
\begin{equation}
\mathcal{L}_{\text{total}} \;=\; \mathcal{L}_{\text{photo}}
\;+\; \lambda_{\alpha}\,\mathcal{L}_{\alpha}
\;+\; \lambda_{\text{leak}}\,\mathcal{L}_{\text{leak}}
\;+\; \lambda_{\text{med}}\,\mathcal{L}_{\text{med}},
\label{eq:loss}
\end{equation}
where $\mathcal{L}_{\text{photo}}$ combines $\ell_1$ and D-SSIM as in 3DGS~\citep{kerbl20233dgs}, $\lambda_{\text{med}} = 0.05$, $\mathcal{L}_{\alpha}$ is an $\ell_1$ loss between the accumulated opacity and the capture's alpha channel ($\lambda_{\alpha}{=}0.2$), $\mathcal{L}_{\text{leak}}$ is the regularizer of Eq.~\ref{eq:erase} ($\lambda_{\text{leak}}{=}0.05$), and $\mathcal{L}_{\text{med}}$ regularizes the learned medium by smoothing \((\log\sigma_t,\text{albedo})\) across neighboring Gaussians and constraining Gaussian scale relative to the transport grid. The scale term is what keeps the two representations compatible. A splat carries one fluence value, read at one point, and one chord length, so a splat spanning many cells sits across regions with different $\phi$ and is described by none of them. We utilize a progressive training strategy for modeling SSS after constructing a fundamental geometry and appearance following SSD-GS~\citep{zheng2026ssdgs}. The subsurface channel is switched on after a $9$k-iteration warm-up, during which the solve is skipped and only the surface terms and the geometry are fit. After it, every step runs one forward and one adjoint solve, and the transport domain is rebuilt whenever densification changes the Gaussian count and in any case every $300$ iterations.

\section{Experiments}
\label{sec:exp}

\subsection{Setup}
\label{sec:setup}

\paragraph{Datasets.}
We evaluate on three OLAT datasets covering subsurface-dominant, opaque, and anisotropic material regimes. The \textbf{SSS-GS} benchmark~\citep{dihlmann2024sssgs}, built for subsurface scattering, provides five synthetic translucent objects and three real light-stage captures. The \textbf{NRHints} set~\citep{zeng2023nrhints} contributes seven real captured objects, most of which are not strongly translucent. The \textbf{GS$^3$} synthetic set~\citep{bi2024gs3} adds six scenes spanning metal, fur, and mixed materials.

\paragraph{Baselines.}
The methods that appear in qualitative comparisons are retrained from public implementations under matched resolution, splits, and training budgets. More details are listed in Appendix~\ref{app:impl}. These are GS$^3$~\citep{bi2024gs3}, a general relighting system, SSD-GS~\citep{zheng2026ssdgs}, a recent relightable Gaussian baseline with an explicit subsurface component that is the local-kernel counterpoint to our solve, RNG~\citep{fan2025rng}, and SSS-GS~\citep{dihlmann2024sssgs}, which introduces a dedicated subsurface scattering component. The per-scene tables in Appendix~\ref{app:perscene} also list 3DGS~\citep{kerbl20233dgs} and GI-GS~\citep{chen2025gigs} with the numbers reported by \citet{zheng2026ssdgs} as additional reference points; both fall well below the retrained baselines. We use a fixed hyperparameter configuration across all scenes, following each baseline's official codebase.

\paragraph{Implementation.}
All timings are measured on one A6000, with the per-scene transfer basis precomputed once and excluded.
VolS-GS renders the five SSS-GS synthetic scenes at $64.5 \pm 4.9$ FPS. On the SSS-GS synthetic scenes, training takes under 3.5 hours and under
8\,GiB of GPU memory per scene. Per-scene frame times, primitive
counts and the rest of the implementation are in Appendix~\ref{app:impl}.

\begin{table}[!htb]\centering
\caption{\textbf{Main results.} Left: the eight subsurface scenes, per scene, every method at the $60$k protocol iteration. Right: the remaining two datasets as test-split means over their $7$ and $6$ scenes (per-scene numbers trained at $100$k in Appendix~\ref{app:perscene}). Best \colorbox{red!25}{red}, second \colorbox{orange!25}{orange}. }
\label{tab:sss}
\resizebox{\textwidth}{!}{%
\begin{tabular}{ll cccccccc | cc}
\toprule
 & Method & \textsc{Bunny} & \textsc{Dragon} & \textsc{Candle} & \textsc{Soap} & \textsc{Statue} & \textsc{Crystal} & \textsc{JoyfulYell} & \textsc{Physalis} & NRHints & GS$^3$ \\
  &        &       &        &        &      &        &         &            &          & \multicolumn{2}{c}{dataset mean} \\
\midrule
\multirow{5}{*}{PSNR$\uparrow$} & GS$^3$ & 34.7109 & 36.1753 & 38.3707 & 38.3713 & 35.3291 & \cellcolor{orange!25}35.2767 & 31.5467 & 28.9163 & 31.94 & \cellcolor{orange!25}31.11 \\
 & SSS-GS & 36.0035 & 34.7397 & 38.3432 & 36.4006 & 35.5485 & 34.6201 & 31.4973 & 28.9331 & 25.16 & 28.23 \\
 & RNG & 35.3300 & 34.6134 & 40.3420 & \cellcolor{orange!25}38.4962 & 29.2148 & 33.6849 & 30.0986 & 28.6492 & 30.74 & 28.64 \\
 & SSD-GS & \cellcolor{orange!25}38.9147 & \cellcolor{orange!25}37.7832 & \cellcolor{orange!25}41.9214 & 38.2959 & \cellcolor{orange!25}35.9034 & 35.2096 & \cellcolor{orange!25}31.8938 & \cellcolor{orange!25}29.5474 & \cellcolor{orange!25}32.05 & 31.01 \\
\cmidrule(lr){2-12}
 & \textbf{VolS-GS (Ours)} & \cellcolor{red!25}40.1453 & \cellcolor{red!25}38.0579 & \cellcolor{red!25}43.6366 & \cellcolor{red!25}41.0098 & \cellcolor{red!25}38.3515 & \cellcolor{red!25}35.6278 & \cellcolor{red!25}32.1856 & \cellcolor{red!25}30.2210 & \cellcolor{red!25}32.37 & \cellcolor{red!25}31.59 \\
\midrule
\multirow{5}{*}{SSIM$\uparrow$} & GS$^3$ & 0.9811 & 0.9769 & 0.9882 & 0.9852 & 0.9815 & \cellcolor{orange!25}0.9813 & 0.9665 & 0.9535 & 0.9500 & \cellcolor{orange!25}0.9676 \\
 & SSS-GS & 0.9797 & 0.9596 & 0.9864 & 0.9762 & 0.9765 & 0.9776 & 0.9603 & 0.9461 & 0.8499 & 0.9452 \\
 & RNG & 0.9859 & 0.9715 & 0.9940 & \cellcolor{orange!25}0.9896 & 0.9647 & 0.9743 & 0.9582 & 0.9483 & 0.9272 & 0.9560 \\
 & SSD-GS & \cellcolor{orange!25}0.9907 & \cellcolor{orange!25}0.9828 & \cellcolor{orange!25}0.9942 & 0.9884 & \cellcolor{orange!25}0.9845 & 0.9808 & \cellcolor{orange!25}0.9676 & \cellcolor{orange!25}0.9585 & \cellcolor{orange!25}0.9504 & 0.9672 \\
\cmidrule(lr){2-12}
 & \textbf{VolS-GS (Ours)} & \cellcolor{red!25}0.9920 & \cellcolor{red!25}0.9835 & \cellcolor{red!25}0.9960 & \cellcolor{red!25}0.9935 & \cellcolor{red!25}0.9899 & \cellcolor{red!25}0.9820 & \cellcolor{red!25}0.9690 & \cellcolor{red!25}0.9590 & \cellcolor{red!25}0.9513 & \cellcolor{red!25}0.9702 \\
\midrule
\multirow{5}{*}{LPIPS$\downarrow$} & GS$^3$ & 0.0239 & 0.0251 & 0.0167 & 0.0243 & 0.0207 & \cellcolor{orange!25}0.0196 & 0.0419 & 0.0479 & \cellcolor{orange!25}0.0795 & 0.0364 \\
 & SSS-GS & 0.0323 & 0.0539 & 0.0222 & 0.0286 & 0.0368 & 0.0231 & 0.0468 & 0.0542 & 0.1589 & 0.0649 \\
 & RNG & 0.0177 & 0.0312 & 0.0080 & \cellcolor{orange!25}0.0116 & 0.0371 & 0.0291 & 0.0492 & 0.0551 & 0.0998 & 0.0465 \\
 & SSD-GS & \cellcolor{orange!25}0.0129 & \cellcolor{orange!25}0.0199 & \cellcolor{orange!25}0.0077 & 0.0186 & \cellcolor{orange!25}0.0185 & 0.0205 & \cellcolor{orange!25}0.0411 & \cellcolor{orange!25}0.0470 & 0.0802 & \cellcolor{orange!25}0.0362 \\
\cmidrule(lr){2-12}
 & \textbf{VolS-GS (Ours)} & \cellcolor{red!25}0.0112 & \cellcolor{red!25}0.0188 & \cellcolor{red!25}0.0060 & \cellcolor{red!25}0.0077 & \cellcolor{red!25}0.0130 & \cellcolor{red!25}0.0189 & \cellcolor{red!25}0.0385 & \cellcolor{red!25}0.0444 & \cellcolor{red!25}0.0779 & \cellcolor{red!25}0.0330 \\
\bottomrule
\end{tabular}}
\end{table}

\begin{figure}[!htb]
    \centering
    \includegraphics[width=0.86\linewidth]{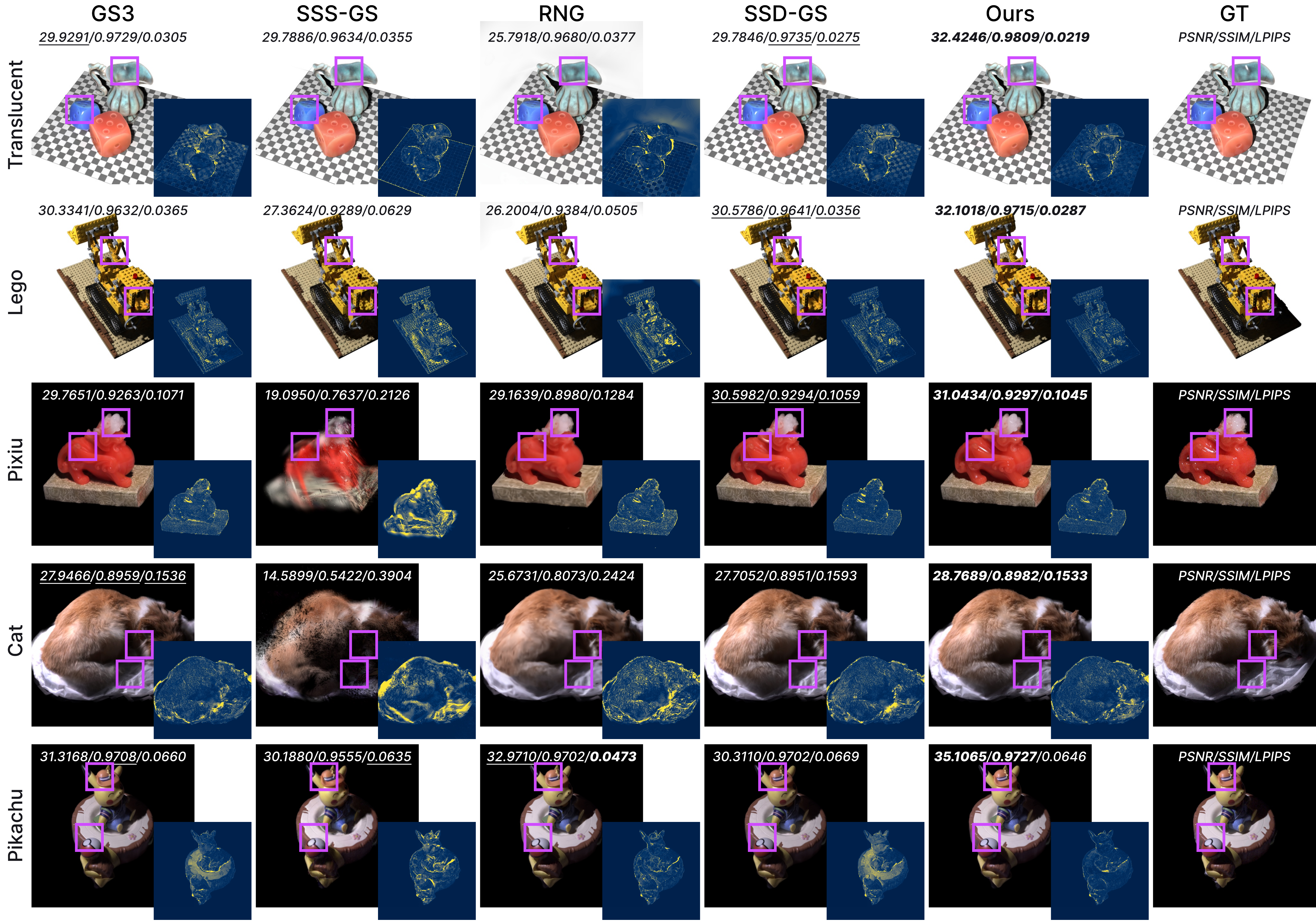}
    \caption{Qualitative relighting comparison under novel lighting across different scenes. \textsc{Translucent} and \textsc{Lego} are synthetic scenes from GS$^3$, while \textsc{Pixiu}, \textsc{Cat}, and \textsc{Pikachu} are real captures from NRHints. Per-frame PSNR/SSIM/LPIPS are printed on each panel, \textbf{best} in bold and \underline{second} underlined within a row. The inset is the error map $\lvert\text{render}-\text{GT}\rvert$ averaged over color channels and clipped at $0.15$, on one ramp shared by every panel.}
    \label{fig:qual}
\end{figure}

\subsection{Main results}
\label{sec:sssbench}

\paragraph{Quantitative Results.} Table~\ref{tab:sss} shows that VolS-GS achieves
the best average performance across all three datasets. All per-scene metrics are provided in Appendix~\ref{app:perscene}. The improvement is most consistent on the SSS-GS benchmark, where subsurface scattering is the dominant effect. The gains also carry over to NRHints and GS$^3$, which contain less translucent and more diverse materials.

\paragraph{Qualitative Results.}
Figure~\ref{fig:qual} shows relighting results across translucent and opaque objects. GS$^3$ and SSS-GS reproduce the overall appearance but their predictions become less accurate around shadows and fine appearance details. SSS-GS in particular does not optimize stably on four of the seven NRHints scenes, and we report and show those rows as measured. RNG also captures the overall illumination but tends to lose high-frequency details, particularly on \textsc{Lego} and \textsc{Cat}. SSD-GS produces competitive results but still struggles with intricate shadow patterns and complex subsurface effects.

\begin{figure}[!htb]
    \centering
    \includegraphics[width=0.80\linewidth]{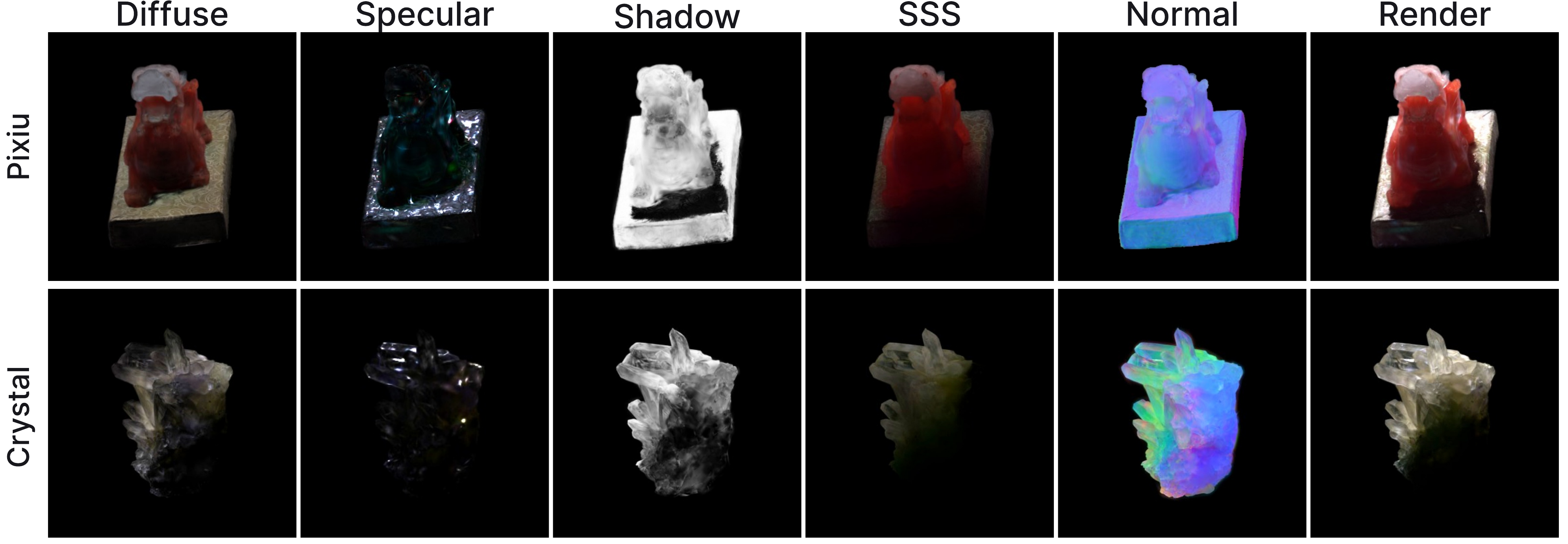}
    \caption{\textbf{Material decomposition.} Our own channels on two real captures,
    \textsc{Pixiu} (top) and \textsc{Crystal} (bottom), one row per scene. The specular column is amplified only where it is already bright ($\times2.5$ above its
    own 95th percentile), identically in both rows.}
    \label{fig:decomp}
\end{figure}

\paragraph{Material Decomposition.}
\label{sec:decomp}
Figure~\ref{fig:decomp} shows the learned decomposition into diffuse, specular, shadow, SSS, and normal components. The SSS channel captures the broad glow visible in backlit regions while the specular and shadow branches remain more localized. More detailed investigations across SSD-GS and SSS-GS are in Appendix~\ref{app:compare}.

\subsection{Ablations}
\label{sec:ablation}

\begin{table}[!htb]
\centering
\caption{\textbf{Ablations} on Bunny and Statue where SSS and shadow are prevalent. Each row retrains the final
configuration with one component changed.}
\label{tab:ablation}
\resizebox{0.85\textwidth}{!}{%
\begin{tabular}{l ccc ccc}
\toprule
 & \multicolumn{3}{c}{\textsc{Bunny}} & \multicolumn{3}{c}{\textsc{Statue}} \\
\cmidrule(lr){2-4}\cmidrule(lr){5-7}
 & PSNR$\uparrow$ & SSIM$\uparrow$ & LPIPS$\downarrow$ & PSNR$\uparrow$ & SSIM$\uparrow$ & LPIPS$\downarrow$ \\
\midrule
\textbf{Full model} & \textbf{40.15} & \textbf{0.9920} & \textbf{0.0112} & \textbf{38.35} & \textbf{0.9899} & \textbf{0.0130} \\
$-$ Volumetric SSS & 34.67 & 0.9812 & 0.0283 & 36.19 & 0.9852 & 0.0173 \\
$-$ Transport cues on shadow & 39.17 & 0.9910 & 0.0126 & 35.09 & 0.9818 & 0.0211 \\
Shadow term multiplied onto SSS & 39.02 & 0.9901 & 0.0132 & 37.66 & 0.9885 & 0.0144 \\
$-$ Shadow term & 36.30 & 0.9859 & 0.0179 & 34.41 & 0.9794 & 0.0230 \\
\bottomrule
\end{tabular}}
\end{table}

\begin{figure}[!t]
    \centering
    \includegraphics[width=0.92\linewidth]{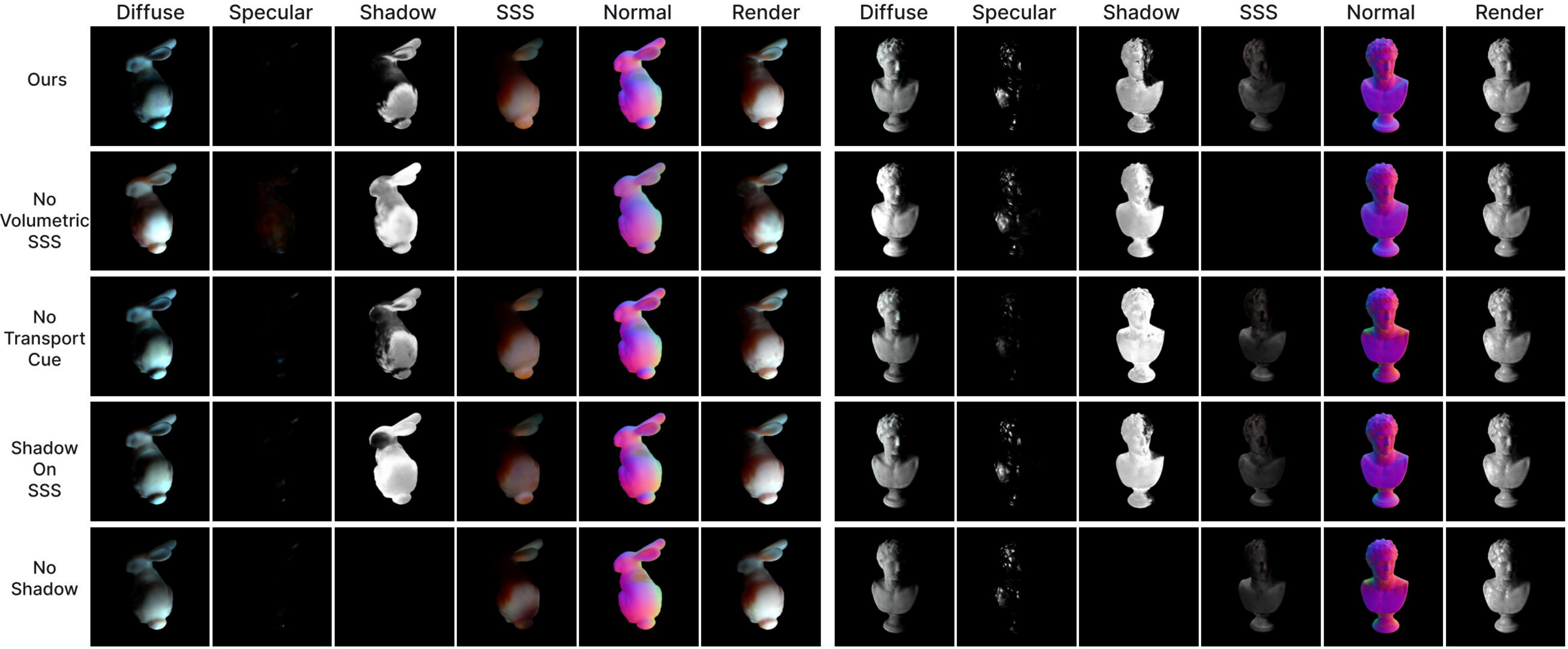}
    \caption{\textbf{Ablation study.} Material decomposition on \textsc{Bunny} (left) and \textsc{Statue} (right) from the SSS-GS dataset. Rows are the full model and the four single-change variants of Table~\ref{tab:ablation}.}
    \label{fig:abl}
\end{figure}

We ablate the main components of VolS-GS to understand their individual roles. We consider four variants: (1) $-$Volumetric SSS, which removes the volumetric scattering component; (2) $-$transport cues on shadow, which prevents the shadow branch from using transport-derived cues; (3) shadow term multiplied onto SSS, which applies the learned shadow to the SSS component as well; and (4) $-$shadow term, which removes the learned shadow component.

Figure~\ref{fig:abl} and Table~\ref{tab:ablation} give the visual and the metric comparison of these variants. Removing the volumetric SSS channel (1) changes the rendered appearance of \textsc{Bunny} and \textsc{Statue} wherever transported light contributes. Removing the transport cues (2) changes the shadow decomposition, most visibly on \textsc{Statue}, where the predicted shadow no longer follows the same attenuation pattern. Multiplying the shadow onto the SSS channel (3) changes the render only slightly but rewrites the shadow channel, degrading its interpretability, and the cost on held-out lights is visible in Table~\ref{tab:ablation}. Removing the shadow term (4) flattens the shading. The specular leakage regularizer contributes to cleaner decomposition, whose investigation is reported in Appendix~\ref{app:specreg}.

\section{Conclusion}
\label{sec:conclusion}

We presented VolS-GS, a relightable Gaussian representation that augments local shading with a learned volumetric subsurface scattering, enabling light propagation across the spatial support of the scene. The shadow term reads cues from the same transport and a prior keeps the specular lobe out of unlit pixels, so the components stay separable and the learned decomposition is interpretable. Across three OLAT benchmarks, VolS-GS achieves high visual quality while reproducing the light scattered inside the object, and the resulting reconstructions relight faithfully under novel lights and views.

\paragraph{Limitations and future work.}
Since our solve for volumetric subsurface scattering is based on a diffusion approximation, it may not suffice where the radiance is still directional, and it
carries no path for light exchanged between surfaces on concave regions rather than through the medium
(Figure~\ref{fig:failure}). We leave directional subsurface scattering in a discrete volume regime and unified representation for diverse non-local effects as interesting
future work.

\clearpage

\bibliography{main}
\bibliographystyle{iclr2027_conference}

\appendix
\clearpage
\section{Related work}
\label{sec:related}

\subsection{Relightable neural and Gaussian representations}
Relighting a captured object means re-rendering it under light it was never photographed under. On
neural fields~\citep{mildenhall2020nerf} one line recovers an explicit reflectance model, pairing
reflectance with visibility~\citep{srinivasan2021nerv}, fitting spherical-Gaussian
BRDFs~\citep{zhang2021physg}, separating shape from material under unknown
illumination~\citep{boss2021nerd,zhang2021nerfactor}, or reparameterising the view-dependent term for
specular surfaces~\citep{verbin2022refnerf}. Another keeps appearance implicit and conditions it on the
light, the setting we work in, learned from point-lit images~\citep{zeng2023nrhints} or carried on a
tensorial field~\citep{jin2023tensoir}.

3D Gaussian splatting~\citep{kerbl20233dgs} moved relighting onto primitives that rasterize in real time. Relightable 3D Gaussians~\citep{gao2023relightable} and
GaussianShader~\citep{jiang2024gaussianshader} give each Gaussian BRDF parameters under environment
lighting. In the one-light-at-a-time (OLAT) setting, OLAT Gaussians~\citep{kuang2024olatgaussians}
condition a network on the light and hint it from a proxy mesh, GS$^3$~\citep{bi2024gs3} adds a Lambertian
lobe with an anisotropic specular basis and a learned residual, RNG~\citep{fan2025rng} drops the explicit
model altogether, and SSD-GS~\citep{zheng2026ssdgs} contributes an MLP-refined shadow field and a
subsurface term. We keep this family's pipeline and protocol. GS$^3$ collects whatever its explicit terms leave out in a single residual, so transported light is
present in its result but cannot be read out as a subsurface component.

\subsection{Non-local light transport in Gaussian scenes}
Inverse rendering on Gaussian primitives recovers geometry, material and lighting jointly:
GS-IR~\citep{liang2024gsir} bakes occlusion into a per-Gaussian representation to make the rendering
equation tractable, and IRGS~\citep{gu2025irgs} evaluates incident radiance on the fly with differentiable
2D Gaussian ray tracing so that secondary bounces are not approximated away. Several recent Gaussian
methods explicitly model light transport between primitives. PRTGS~\citep{guo2024prtgs} precomputes the radiance transfer
of Gaussian splats so that soft shadows and indirect illumination can be evaluated in real time under low-frequency
environment lighting. Ref-Gaussian~\citep{yao2025refgaussian} adds physically based deferred shading and
Gaussian-grounded inter-reflection for glossy objects. Radiosity-GS~\citep{jiang2025radiositygs} makes a radiosity
solve differentiable on Gaussian surfels and evaluates it under the same point-light capture we use. These methods transport light between surfaces, whereas subsurface scattering transports light through the interior of the object.

\subsection{Subsurface scattering}
A BSSRDF integrates over every point at which light could have entered. Analytic kernels assume a
homogeneous, semi-infinite medium under uniform illumination, collapsing that integral into a function of
the entry-to-exit distance~\citep{hanrahan1993reflection,jensen2001bssrdf,donner2005light,deon2011quantized,habel2013photon}.
Locality is a property of the evaluation, not of these models: they also integrate the entry surface
hierarchically~\citep{jensen2002rapid}, restore the dropped directional
dependence~\citep{frisvad2014directional}, and fit profiles for the same
integral~\citep{christensen2015approximate}. 

For spatially varying media the field turned to numerical solvers, from the diffusion equation on blobs
and grids~\citep{stam1995multiple} to finite differences on a PolyGrid~\citep{wang2008modeling} and a
finite-element treatment with a diffusive-source boundary~\citep{arbree2011heterogeneous}. Ours is that
equation, discretized on the Gaussian scene's own occupancy. A separate line recovers the medium
from images instead of assuming
it~\citep{gkioulekas2013inverse,che2020towards,deng2022reconstructing}.

A newer line learns subsurface appearance instead of solving for it, relaxing the planar
boundary~\citep{vicini2019learned}, rendering object scattering functions in real
time~\citep{yu2023kiloosf}, fitting scattering-aware Gaussians~\citep{yuan2026gtsr}, or pairing explicit
shading with an implicit residual~\citep{dihlmann2024sssgs}, which also contributes our benchmark. SSD-GS~\citep{zheng2026ssdgs} evaluates a dipole at each Gaussian. In VolS-GS the
coefficients are still learned for each primitive. Where the light goes is decided by the solve.

\section{Theoretical properties}
\label{app:theory}

This appendix states the properties referenced in Sections~\ref{sec:sss1} and~\ref{sec:sss2}. Throughout, $\sigma_t = \sigma_s + \sigma_a$ and $\kappa = 1/(3\sigma_t)$ are as defined in Eq.~\ref{eq:medium}, $\alpha = \sigma_s/\sigma_t \in (0,1)$ is the single-scattering albedo, and $\mathbf{J} = -\kappa\nabla\phi$ denotes the $P_1$ flux vector, not the scalar source $Q$ of Eq.~\ref{eq:single}. Phase is isotropic. An anisotropic phase reduces at the diffusion level to $\sigma_s' = (1-g)\sigma_s$~\citep{stam1995multiple,ishimaru1978wave}, so $g$ is not separately identifiable from the reduced coefficient our network already predicts; we therefore fix $g = 0$.

\begin{proposition}[Well-posedness and stability of the discrete solve]
\label{prop:spd}
Let $F = D + M + S$ be the finite-volume operator of Eq.~\ref{eq:fv}: $D$ the divergence part with harmonic-mean face conductances $\kappa_{cd} > 0$ weighted by the open face fractions $\alpha_{cd} \ge 0$, $M = \mathrm{diag}\big(\sigma_a(c)\,O(c)\big)$ with $\sigma_a > 0$, and $S$ the Robin boundary part with nonnegative diagonal entries. Then $F$ is symmetric positive definite, the discrete solution $\phi = F^{-1} b$ exists and is unique, and $\|\phi\| \le \|b\| / \lambda_{\min}(M)$.
\end{proposition}
\begin{proof}
$D$ is a weighted graph Laplacian on the active cells ($\mathbf{v}^\top D \mathbf{v} = \sum_{(c,d)} \kappa_{cd}\,\alpha_{cd}\,(v_c - v_d)^2 / h^2 \ge 0$) and $S$ is diagonal and nonnegative. $M$ is diagonal with strictly positive entries because the single-scattering albedo $\alpha = \sigma_s/\sigma_t$ is capped at $0.999$ while $\sigma_t \ge \sigma_{\min}$, hence $\sigma_a = (1-\alpha)\,\sigma_t \ge (1 - 0.999)\,\sigma_{\min} > 0$. The sum is therefore symmetric positive definite, the CG iteration is well defined and convergent, and $\lambda_{\min}(F) \ge \lambda_{\min}(M)$ bounds the solve.
\end{proof}

\begin{proposition}[Adjoint gradients for the full operator]
\label{prop:adjoint}
For a scalar loss $\mathcal{L}(\phi)$ with $\phi = F^{-1} b$, the parameter gradients are
$\partial_b \mathcal{L} = F^{-1} \nabla_\phi \mathcal{L}$ and, for any parameter $\theta$ entering $F$,
$\partial_\theta \mathcal{L} = -\,(F^{-1} \nabla_\phi \mathcal{L})^\top (\partial_\theta F)\, \phi$.
Since $F$ is symmetric, the adjoint system is the same operator, so backpropagation through the solve costs one additional CG solve.
\end{proposition}

\begin{proof}
Differentiating $F\phi = b$ gives $d\phi = F^{-1}(db - (dF)\phi)$, and setting $\boldsymbol{\lambda} = F^{-\top}\nabla_\phi\mathcal{L} = F^{-1}\nabla_\phi\mathcal{L}$ (symmetry) gives both expressions. This is the standard adjoint-state method.
\end{proof}

\begin{lemma}[Reciprocity in heterogeneous media]
\label{lem:recip}
The variable-coefficient diffusion operator $\mathcal{A}\phi = -\nabla\!\cdot\!(\kappa(\mathbf{x})\nabla\phi) + \sigma_a(\mathbf{x})\phi$ with the Robin condition of Eq.~\ref{eq:pde} is self-adjoint; its Green's function satisfies $G(\mathbf{x}, \mathbf{x}') = G(\mathbf{x}', \mathbf{x})$ for arbitrary $\kappa > 0$, $\sigma_a \ge 0$.
\end{lemma}
\begin{proof}
For $\phi, \psi$ satisfying the Robin condition,
$\langle \mathcal{A}\phi, \psi\rangle - \langle \phi, \mathcal{A}\psi\rangle
= \oint \kappa\,(\phi\,\partial_n\psi - \psi\,\partial_n\phi)\, dS
= \oint \tfrac{1}{2A}(\phi\psi - \psi\phi)\, dS = 0$,
using Green's identity and then the boundary condition to substitute $\kappa\,\partial_n \phi = -\phi/(2A)$. The zeroth-order term is symmetric trivially, and symmetry of $G$ follows in the standard way. Our discretization preserves it exactly (Proposition~\ref{prop:spd}: $F = F^\top$).
\end{proof}

\begin{remark}[Superposition over lights]
The diffusion equation is linear in the incident source and $b$ is linear in the incident field, so several lights superpose, and the learned components preserve this at the level of the operator: $b$ is assembled per light, the solve is applied to their sum, and the visibility decay multiplies each light's own ballistic term. Two qualifications. The split into single-scattered ($L_1$) and diffused ($L_{\mathrm{ms}}$) radiance is not strictly disjoint, because the source we assemble is not albedo-reduced (Section~\ref{sec:sss2}); this does not affect superposition, since both terms are built per light. And the material coefficients are conditioned on direction (Section~\ref{sec:sss1}), so superposition is exact for the transport given the medium, not for the medium itself.
\end{remark}

\begin{remark}[Degeneration as $\sigma_s \to 0$]
\label{rem:degen}
As $\sigma_s \to 0$ the source $b = \sigma_s E O$ of Eq.~\ref{eq:source} vanishes linearly, so
$\phi \to 0$ and $L_{\mathrm{ms}} \to 0$, while the operator $F = D + M + S$ stays symmetric
positive definite throughout: the channel switches itself off continuously rather than the solve
degenerating. On opaque materials the surface terms of Eq.~\ref{eq:skeleton} are therefore left
unperturbed.
\end{remark}

\begin{remark}[Exit coefficient]
The exit factor $(1 + 1/A(\eta))/(4\pi)$ in Eq.~\ref{eq:exit} follows from substituting the Robin relation $\kappa\,\partial_n\phi = -\phi/(2A)$ into the first-order angular expansion $L(\boldsymbol{\omega}) = \frac{1}{4\pi}(\phi + 3\,\mathbf{J}\cdot\boldsymbol{\omega})$ with $\mathbf{J} = -\kappa\nabla\phi$, evaluated at the boundary for outgoing directions and averaged over the outgoing hemisphere with a cosine weight; it is the $\Gamma_s = 0$ specialization of the render query of \citet{arbree2011heterogeneous}. At $\eta = 1$, $A = 1$ and the factor is $1/(2\pi)$. The exit carries no further factor of $\sigma_s$: it already enters through the source, and $\phi$ is linear in it, so applying it again would make the exitant radiance quadratic in a chromatic coefficient.
\end{remark}

\section{Implementation details}
\label{app:impl}

\paragraph{Per-frame cost.}
Per frame the transport adds a source march, one finite-volume solve and the $32$-sample single-scatter quadrature.
The operator has one row per active grid cell, between $5.1\times10^{4}$ and $1.4\times10^{5}$ cells across the $21$
scenes (median $7.4\times10^{4}$ of the $64^3$ total). At inference a precomputed transfer basis~\citep{sloan2002prt} over a Fibonacci light
shell warm-starts the solve, so eight preconditioned conjugate-gradient iterations suffice. Measured end to end on
four scenes, test PSNR differs from the converged solve by $-0.009$ to $+0.000$\,dB, and every number this paper
reports is produced with that setting. The fixed-topology parts of the pipeline are replayed as CUDA graphs. Because
the grid is $64^3$ for every scene, frame time follows the primitive count rather than the pixel count, so a
sixteen-fold increase in pixels at a matched primitive count (\textsc{Bunny} and \textsc{CupFabric}, both near
$82$k) adds only a few milliseconds while the densest scene (\textsc{Drums}, $550$k) costs $47.07$\,ms. Training runs on a single
GPU with the same schedule as the baselines, with the calibration gain trained for the first $20$k iterations and
then frozen.

\begin{table}[!htb]\centering\footnotesize
\caption{\textbf{Render cost.} A single run on one A6000, reported per dataset and per resolution.}
\label{tab:fps}
\begin{tabular}{llrr}
\toprule
dataset & res. & scenes & FPS \\
\midrule
SSS-GS synthetic & $256^2$  & 5 & $64.5 \pm 4.9$ \\
SSS-GS real      & $512^2$  & 3 & $66.0 \pm 10.2$ \\
NRHints real     & $512^2$  & 4 & $46.0 \pm 2.5$ \\
                 & $1024^2$ & 3 & $54.8 \pm 11.7$ \\
GS$^3$ synthetic & $512^2$  & 6 & $32.9 \pm 11.3$ \\
\bottomrule
\end{tabular}

\vspace{6pt}
\begin{tabular}{lrrrrrr}
\toprule
 & \multicolumn{2}{c}{\textsc{Bunny} ($82$k)} & \multicolumn{2}{c}{\textsc{Statue} ($108$k)}
 & \multicolumn{2}{c}{\textsc{Drums} ($550$k)} \\
\cmidrule(lr){2-3}\cmidrule(lr){4-5}\cmidrule(lr){6-7}
stage & ms & \% & ms & \% & ms & \% \\
\midrule
visibility rasterization & $3.48$ & $23.5$ & $3.64$ & $22.4$ & $8.18$ & $17.4$ \\
shading, pre-solve       & $0.96$ & $6.5$ & $1.25$ & $7.7$ & $7.20$ & $15.3$ \\
transport solve          & $8.19$ & $55.3$ & $8.92$ & $54.7$ & $23.81$ & $50.6$ \\
composition, post-solve  & $0.38$ & $2.5$ & $0.49$ & $3.0$ & $2.41$ & $5.1$ \\
primary rasterization    & $1.80$ & $12.2$ & $1.98$ & $12.2$ & $5.46$ & $11.6$ \\
\midrule
frame                    & $14.81$ & $100$ & $16.29$ & $100$ & $47.07$ & $100$ \\
\bottomrule
\end{tabular}
\end{table}

\begin{algorithm}[t]
\caption{One transport evaluation in VolS-GS. Steps 1--2 run when densification changed the Gaussian count and in any case every $300$ iterations.}
\label{alg:transport}
\begin{algorithmic}[1]
\Require Gaussians $\{(\mathbf{x}_i, \mathbf{q}_i, \mathbf{s}_i, o_i, \mathbf{z}_i)\}$, light position $\mathbf{p}_\ell$, camera
\Ensure per-Gaussian subsurface radiance $L^{\mathrm{sss}}_i$
\State $\Omega \gets \{$cells whose splatted occupancy exceeds the threshold$\}$ \Comment{$64^3$ grid over the scene bounds}
\State splat $(\boldsymbol{\kappa}, \boldsymbol{\sigma}_a, \boldsymbol{\sigma}_s)$ to cell centers; fill interior cells from the nearest surface cell
\ForAll{$c \in \Omega$} \Comment{source, Eq.~\ref{eq:source}}
  \State $\tau_c \gets$ occupancy integrated over $32$ steps from $c$ toward $\mathbf{p}_\ell$
  \State $b_c \gets \sigma_{s,c}\, O_c \exp(-\sigma_{t,c}\tau_c)\,/\,d_c^2$
\EndFor
\State assemble $F = D + M + S$ \Comment{face conductances, absorption, Robin boundary}
\State $\phi \gets \mathrm{CG}(F, b)$ \Comment{warm-started from the transfer basis at inference}
\State $\phi_i \gets$ trilinear sample of $\phi$ at $\mathbf{x}_i$ \Comment{pull back to primitives}
\State $L^{\mathrm{sss}}_i \gets \mathbf{k}^{\mathrm{sss}}_i \odot g_{\mathrm{cal}}\,\big(\,\mathrm{exit}(\phi_i) + \mathrm{singleScatter}(\mathbf{x}_i, \mathbf{p}_\ell)\,\big)$
\Statex \textbf{backward:} solve $F\lambda = \partial\mathcal{L}/\partial\phi$ with the same operator; $\lambda$ gives the gradients of $b$ and of the diagonal of $M+S$. Face conductances are stop-gradient, and the solver iterations are never unrolled.
\end{algorithmic}
\end{algorithm}

\paragraph{Transport implementation.}
Algorithm~\ref{alg:transport} states the steps. The grid is rebuilt whenever densification changes the Gaussian count and in any case every
$300$ iterations, and the interface-area estimate entering the Robin term is box-smoothed over
$3^3$ cells and renormalized over the active cells, the standard treatment for a surface term on a
staircase grid~\citep{brackbill1992csf}. Without it the solve is not grid convergent: $\phi$ at $64^3$ and
at $96^3$ agrees to a correlation of $0.08$, and to $0.97$ with it, while the source assembled from the
same fields already agrees to $0.96$ at both resolutions, which places the discrepancy in the operator
rather than in the splatting or the march. Training runs $20$ conjugate-gradient iterations preconditioned by a three-level geometric V-cycle
\citep{stam1995multiple} to a tolerance of $10^{-3}$, and both the forward and the adjoint solve are truncated, so
the gradient is the implicit-function estimator evaluated at an approximate solution. The source attenuation is factorised as $\exp(-\sigma_t(c)\,\tau(c))$ rather than integrated along the ray.
The march accumulates occupancy only, so $\tau$ is a geometric quantity that survives until the grid is
rebuilt, while $\sigma_t$ is refreshed at every step; a path-integrated extinction would have to be
remarched for every cell and light at every iteration. The factorisation is exact when the medium is
uniform along the path and approximate otherwise.

\paragraph{Material head and shading.}
The material head is a small MLP over the positional encoding, the per-Gaussian latent, and the frequency-encoded light and view directions; outputs map to $[\sigma_{\min}, \sigma_{\max}] = [10^{-3}, 25]$ per channel through a scaled sigmoid. We set the exit index to $\eta = 1.0$ on every scene, including the real captures, so the Fresnel factors reduce to identity and the residual boundary scale is absorbed by the calibration above.
The specular lobe is a mixture of $8$ anisotropic spherical Gaussians with scene-global lobe frames and widths and per-Gaussian amplitudes and specular color, evaluated in the Gaussian's local frame with a Schlick Fresnel factor~\citep{schlick1994inexpensive}. The decay MLP of Eq.~\ref{eq:decay} has three hidden layers of $32$ units.

\paragraph{Training.}
Losses, schedules, densification, and optimizer settings are inherited unchanged from the base relighting system, plus the scale-consistency penalty (weight $0.05$) that discourages single splats from spanning many transport cells. A single configuration is used for all scenes. Retrained baselines follow the same per-dataset schedule. Every number in the paper, ours and every baseline's, is read at the protocol iteration for its dataset (60k or 100k), which each table states.

\section{Additional qualitative results}
\label{app:addqual}

\subsection{Per-scene results}
\label{app:perscene}

Tables~\ref{tab:real} and~\ref{tab:gs3syn} give the per-scene numbers behind the dataset means of Table~\ref{tab:sss}; the SSS-GS family is already per scene in Table~\ref{tab:sss}.
Figure~\ref{fig:addqual} shows five further test frames beside the tables, chosen to cover all three dataset families rather than to favour any one of them.

\begin{table}[!htb]\centering
\caption{\textbf{Real captures, per scene} (NRHints). The 3DGS and GI-GS rows are quoted from \citet{zheng2026ssdgs}. Every other row is our own retraining to the $100$k protocol iteration. Best \colorbox{red!25}{red}, second \colorbox{orange!25}{orange} among retrained rows. }
\label{tab:real}
\resizebox{\textwidth}{!}{%
\begin{tabular}{llccccccc}
\toprule
 & Method & \textsc{Cat} & \textsc{CatSmall} & \textsc{CupFabric} & \textsc{Fish} & \textsc{FurScene} & \textsc{Pikachu} & \textsc{Pixiu} \\
\midrule
\multirow{7}{*}{PSNR$\uparrow$} & 3DGS & 14.5326 & 22.5727 & 25.0488 & 22.8411 & 18.4838 & 19.6310 & 18.5501 \\
 & GI-GS & 13.9988 & 22.3222 & 24.3821 & 22.7500 & 17.8520 & 19.1867 & 18.1064 \\
 & GS$^3$ & 27.7597 & 34.6288 & 36.3707 & \cellcolor{orange!25}31.1387 & 30.5467 & \cellcolor{orange!25}32.6480 & 30.4776 \\
 & SSS-GS & 16.6533 & 31.4044 & 32.9904 & 22.8345 & 20.4048 & 28.4206 & 23.4290 \\
 & RNG & 26.7680 & 34.2488 & \cellcolor{red!25}37.1390 & 29.2007 & 27.8608 & 31.1881 & 28.7688 \\
 & SSD-GS & \cellcolor{orange!25}27.9527 & \cellcolor{orange!25}34.6987 & 36.9159 & 31.0531 & \cellcolor{orange!25}30.6964 & 32.1040 & \cellcolor{orange!25}30.9603 \\
\cmidrule(lr){2-9}
 & \textbf{VolS-GS (Ours)} & \cellcolor{red!25}28.1084 & \cellcolor{red!25}35.0647 & \cellcolor{orange!25}37.0219 & \cellcolor{red!25}31.1809 & \cellcolor{red!25}30.8567 & \cellcolor{red!25}32.9575 & \cellcolor{red!25}31.3762 \\
\midrule
\multirow{7}{*}{SSIM$\uparrow$} & 3DGS & 0.6962 & 0.8896 & 0.9430 & 0.8312 & 0.7869 & 0.9000 & 0.8298 \\
 & GI-GS & 0.3162 & 0.8750 & 0.9178 & 0.7437 & 0.5811 & 0.8724 & 0.6117 \\
 & GS$^3$ & 0.9059 & \cellcolor{orange!25}0.9766 & 0.9818 & \cellcolor{red!25}0.9254 & 0.9503 & \cellcolor{orange!25}0.9678 & 0.9424 \\
 & SSS-GS & 0.6475 & 0.9464 & 0.9625 & 0.8065 & 0.7868 & 0.9477 & 0.8516 \\
 & RNG & 0.8464 & 0.9677 & \cellcolor{red!25}0.9827 & 0.8933 & 0.9163 & 0.9657 & 0.9180 \\
 & SSD-GS & \cellcolor{orange!25}0.9069 & 0.9760 & \cellcolor{orange!25}0.9825 & 0.9246 & \cellcolor{orange!25}0.9511 & 0.9668 & \cellcolor{orange!25}0.9446 \\
\cmidrule(lr){2-9}
 & \textbf{VolS-GS (Ours)} & \cellcolor{red!25}0.9084 & \cellcolor{red!25}0.9772 & 0.9823 & \cellcolor{red!25}0.9254 & \cellcolor{red!25}0.9522 & \cellcolor{red!25}0.9689 & \cellcolor{red!25}0.9447 \\
\midrule
\multirow{7}{*}{LPIPS$\downarrow$} & 3DGS & 0.3033 & 0.1183 & 0.0868 & 0.1814 & 0.2101 & 0.1103 & 0.1721 \\
 & GI-GS & 0.3518 & 0.1339 & 0.1116 & 0.2063 & 0.2492 & 0.1347 & 0.3089 \\
 & GS$^3$ & 0.1314 & 0.0650 & 0.0525 & \cellcolor{orange!25}0.0872 & \cellcolor{orange!25}0.0707 & 0.0673 & 0.0821 \\
 & SSS-GS & 0.3400 & 0.0881 & 0.0688 & 0.1809 & 0.2167 & 0.0744 & 0.1434 \\
 & RNG & 0.1992 & \cellcolor{orange!25}0.0632 & \cellcolor{red!25}0.0371 & 0.1292 & 0.1099 & \cellcolor{red!25}0.0521 & 0.1082 \\
 & SSD-GS & \cellcolor{orange!25}0.1313 & 0.0679 & 0.0519 & 0.0879 & 0.0737 & 0.0680 & \cellcolor{orange!25}0.0807 \\
\cmidrule(lr){2-9}
 & \textbf{VolS-GS (Ours)} & \cellcolor{red!25}0.1297 & \cellcolor{red!25}0.0630 & \cellcolor{orange!25}0.0516 & \cellcolor{red!25}0.0858 & \cellcolor{red!25}0.0701 & \cellcolor{orange!25}0.0658 & \cellcolor{red!25}0.0790 \\
\bottomrule
\end{tabular}}
\end{table}

\begin{figure}[p]
    \centering
    \includegraphics[width=\linewidth]{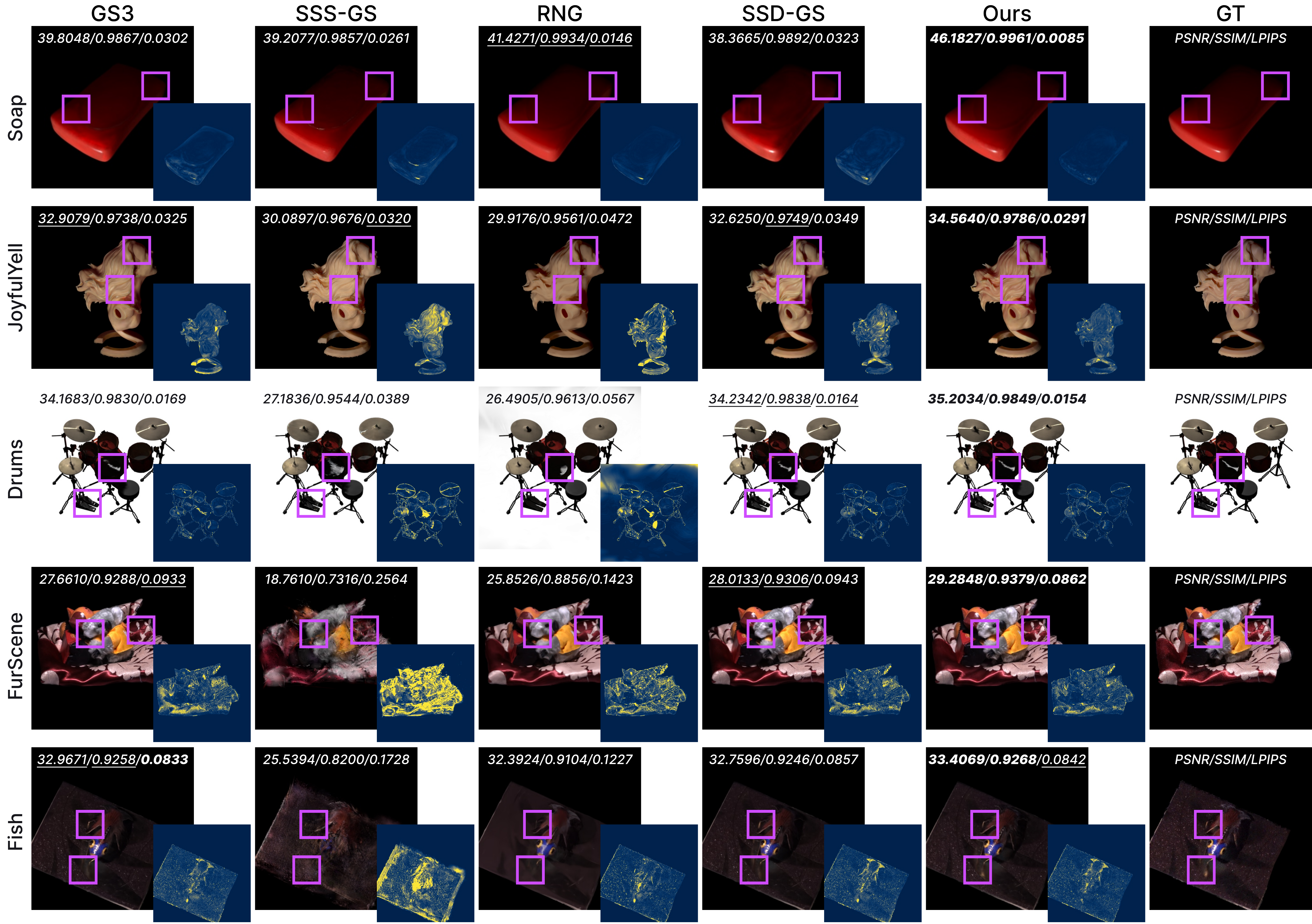}
    \caption{\textbf{Additional qualitative comparisons.} Five further test frames, one per row: \textsc{Soap} and \textsc{JoyfulYell} from the subsurface set, \textsc{Drums} from the GS$^3$ synthetic set, \textsc{FurScene} and \textsc{Fish} real captures. Per-frame PSNR/SSIM/LPIPS are printed on each panel, \textbf{best} bold and \underline{second} underlined; the inset is $\lvert\text{render}-\text{GT}\rvert$ clipped at $0.15$ on one shared ramp.}
    \label{fig:addqual}
\end{figure}

\begin{table}[p]\centering
\caption{\textbf{GS$^3$ synthetic scenes, per scene} trained at $100$K iterations. Notation as in Table~\ref{tab:real}.}
\label{tab:gs3syn}
\resizebox{\textwidth}{!}{%
\begin{tabular}{llcccccc}
\toprule
 & Method & \textsc{Translucent} & \textsc{AnisoMetal} & \textsc{Drums} & \textsc{FurBall} & \textsc{Hotdog} & \textsc{Lego} \\
\midrule
\multirow{7}{*}{PSNR$\uparrow$} & 3DGS & 16.4899 & 17.1009 & 24.5093 & 20.1206 & 16.9535 & 15.9886 \\
 & GI-GS & 16.0766 & 15.9567 & 24.6177 & 19.5295 & 16.8118 & 16.4229 \\
 & GS$^3$ & 29.7079 & \cellcolor{orange!25}28.5958 & 32.9349 & 34.0920 & \cellcolor{red!25}32.2632 & \cellcolor{orange!25}29.0746 \\
 & SSS-GS & 28.2425 & 24.2328 & 28.4225 & 31.0156 & 30.8217 & 26.6436 \\
 & RNG & 27.1813 & 26.1441 & 29.7962 & 33.0337 & 29.8325 & 25.8664 \\
 & SSD-GS & \cellcolor{orange!25}29.9792 & 28.2052 & \cellcolor{orange!25}33.0411 & \cellcolor{orange!25}34.1532 & 31.6721 & 28.9808 \\
\cmidrule(lr){2-8}
 & \textbf{VolS-GS (Ours)} & \cellcolor{red!25}31.6353 & \cellcolor{red!25}28.6080 & \cellcolor{red!25}33.2051 & \cellcolor{red!25}34.2224 & \cellcolor{orange!25}31.7383 & \cellcolor{red!25}30.1361 \\
\midrule
\multirow{7}{*}{SSIM$\uparrow$} & 3DGS & 0.8958 & 0.8849 & 0.9439 & 0.8951 & 0.8599 & 0.7904 \\
 & GI-GS & 0.8586 & 0.8304 & 0.8941 & 0.8592 & 0.8282 & 0.7647 \\
 & GS$^3$ & 0.9710 & \cellcolor{orange!25}0.9627 & 0.9838 & 0.9683 & \cellcolor{red!25}0.9724 & \cellcolor{orange!25}0.9474 \\
 & SSS-GS & 0.9562 & 0.9249 & 0.9619 & 0.9534 & 0.9616 & 0.9134 \\
 & RNG & 0.9659 & 0.9509 & 0.9719 & 0.9587 & 0.9622 & 0.9267 \\
 & SSD-GS & \cellcolor{orange!25}0.9724 & 0.9606 & \cellcolor{orange!25}0.9838 & \cellcolor{red!25}0.9703 & 0.9711 & 0.9447 \\
\cmidrule(lr){2-8}
 & \textbf{VolS-GS (Ours)} & \cellcolor{red!25}0.9783 & \cellcolor{red!25}0.9630 & \cellcolor{red!25}0.9839 & \cellcolor{orange!25}0.9702 & \cellcolor{orange!25}0.9719 & \cellcolor{red!25}0.9539 \\
\midrule
\multirow{7}{*}{LPIPS$\downarrow$} & 3DGS & 0.0748 & 0.0704 & 0.0442 & 0.0865 & 0.1128 & 0.1416 \\
 & GI-GS & 0.1155 & 0.1179 & 0.1242 & 0.1694 & 0.1638 & 0.1643 \\
 & GS$^3$ & 0.0319 & \cellcolor{orange!25}0.0348 & \cellcolor{orange!25}0.0160 & 0.0529 & \cellcolor{orange!25}0.0340 & \cellcolor{orange!25}0.0490 \\
 & SSS-GS & 0.0521 & 0.0677 & 0.0421 & 0.0938 & 0.0472 & 0.0862 \\
 & RNG & 0.0336 & 0.0412 & 0.0273 & 0.0730 & 0.0423 & 0.0618 \\
 & SSD-GS & \cellcolor{orange!25}0.0304 & 0.0363 & 0.0165 & \cellcolor{orange!25}0.0466 & 0.0360 & 0.0515 \\
\cmidrule(lr){2-8}
 & \textbf{VolS-GS (Ours)} & \cellcolor{red!25}0.0246 & \cellcolor{red!25}0.0345 & \cellcolor{red!25}0.0155 & \cellcolor{red!25}0.0463 & \cellcolor{red!25}0.0335 & \cellcolor{red!25}0.0438 \\
\bottomrule
\end{tabular}}
\end{table}

\FloatBarrier
\subsection{Relighting under a light sweep}
\label{app:sweep}

\begin{figure}[p]
    \centering
    \includegraphics[width=0.95\linewidth]{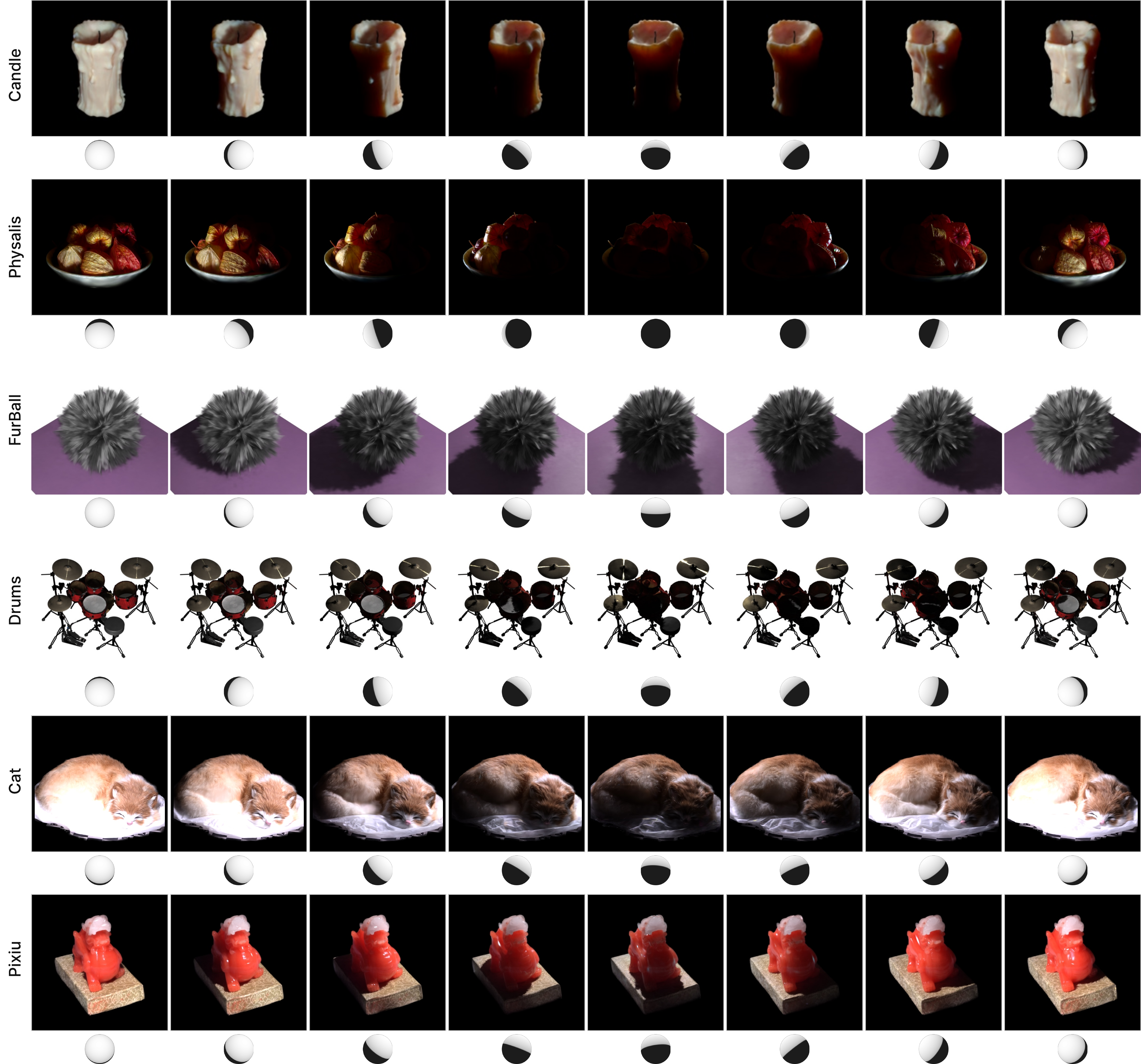}
    \caption{\textbf{Relighting under a continuous light sweep.} Eight unseen light positions per scene, one per column, with the disc below each column giving its light direction.}
    \label{fig:sweep}
\end{figure}

A fixed test split cannot show whether appearance stays coherent as the light moves, only whether it is right at the frames that were held out. Figure~\ref{fig:sweep} sweeps the light through a full revolution at the elevation the training lights actually occupy, so every sample sits inside the training distribution in elevation and differs only in azimuth. The cast shadow and the subsurface glow remain spatially consistent as the light moves, an additional qualitative check on the learned decomposition.

\subsection{Environment map relighting}
\label{app:envmap}

\begin{figure}[p]
    \centering
    \includegraphics[width=\linewidth]{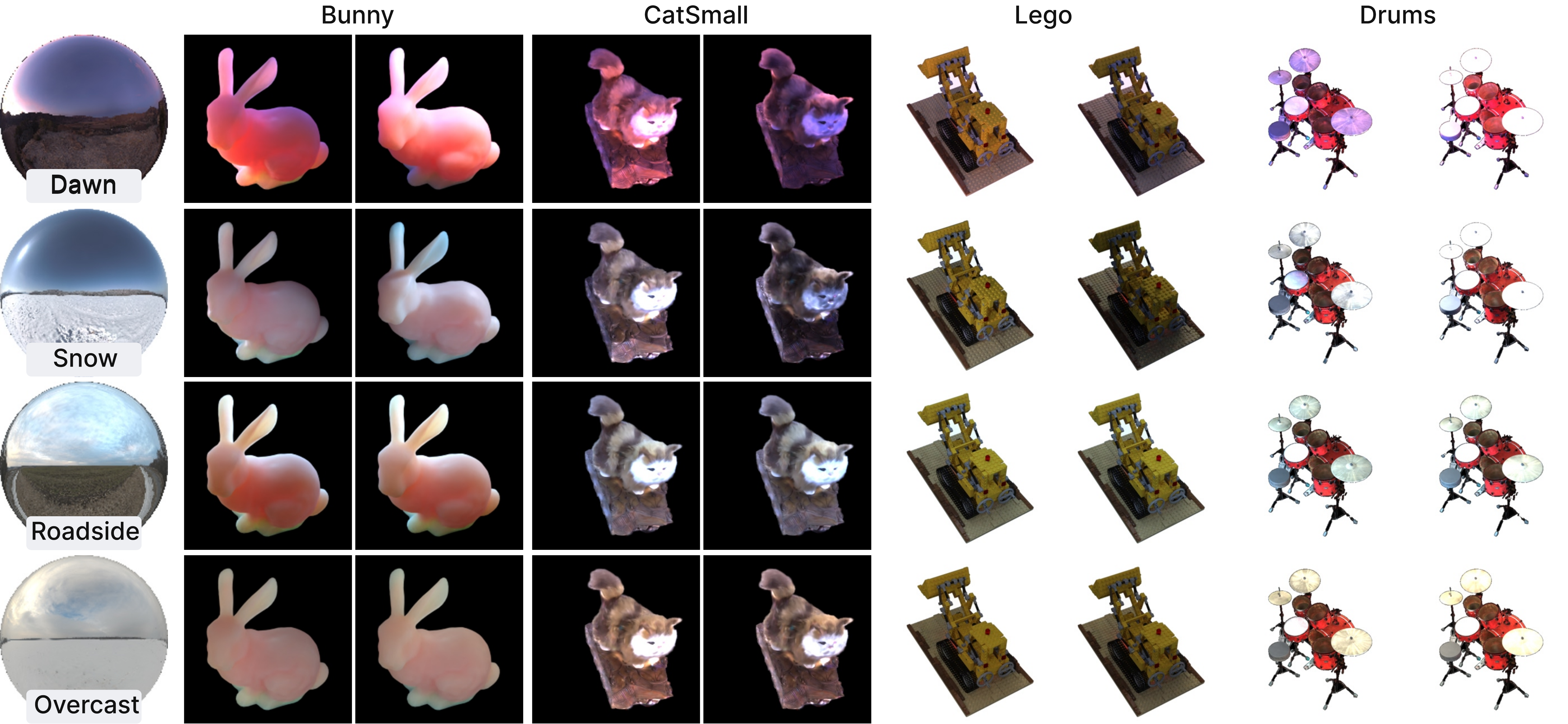}
    \caption{\textbf{Relighting under an environment map.} One row per map, shown as a mirror ball at the left; two frames of a full rotation of the map for each of four scenes. Nothing is retrained.}
    \label{fig:envmap}
\end{figure}

Our setting is one light at a time at a known position, and the source term of Eq.~\ref{eq:source} is written for exactly that. Relighting under an environment map is a different problem: there the illumination is unknown and has to be estimated alongside the material. We do not evaluate on those benchmarks and we do not compare against methods built for them.

The decomposition itself is not tied to a single light, however. An environment map is a sum of directional sources, and the solve is linear in the source, so relighting under one is the same computation repeated: we importance-sample $K=256$ directions from the map, place a distant light along each, and add the per-light renders with the map's own weights. Figure~\ref{fig:envmap} shows the result on four maps, with the map rotated about the vertical axis between the two frames of each pair. Nothing is retrained.

\section{Ablations on the remaining scenes}
\label{app:ablmore}
\begin{table}[!htb]\centering\footnotesize
\setlength{\tabcolsep}{4pt}
\caption{\textbf{Ablations on further scenes} (test PSNR$\uparrow$). The same five rows as Table~\ref{tab:ablation}, whose \textsc{Statue} and \textsc{Bunny} PSNR columns are repeated here, extended to five further scenes. Every arm is read at the protocol iteration for its family, $60$k on the SSS-GS scenes and $100$k elsewhere.}
\label{tab:ablmore}
\resizebox{\textwidth}{!}{%
\begin{tabular}{lccccccc}
\toprule
 & \textsc{Statue} & \textsc{Bunny} & \textsc{JoyfulYell} & \textsc{Translucent} & \textsc{Drums} & \textsc{CupFabric} & \textsc{Pikachu} \\
\midrule
\textbf{Full model} & \cellcolor{red!25}38.35 & \cellcolor{red!25}40.15 & \cellcolor{red!25}32.19 & \cellcolor{red!25}31.64 & \cellcolor{red!25}33.21 & \cellcolor{red!25}37.02 & \cellcolor{red!25}32.96 \\
$-$ Volumetric SSS & 36.19 & 34.67 & 31.54 & 30.88 & \cellcolor{orange!25}33.11 & 35.78 & 31.43 \\
$-$ Transport cues on shadow & 35.09 & \cellcolor{orange!25}39.17 & \cellcolor{orange!25}32.08 & 31.01 & 32.74 & 36.41 & 32.19 \\
Shadow term multiplied onto SSS & \cellcolor{orange!25}37.66 & 39.02 & 31.87 & \cellcolor{orange!25}31.15 & 32.93 & \cellcolor{orange!25}36.82 & \cellcolor{orange!25}32.46 \\
$-$ Shadow term & 34.41 & 36.30 & 29.24 & 28.42 & 32.47 & 34.07 & 28.81 \\
\bottomrule
\end{tabular}}
\end{table}

Table~\ref{tab:ablation} reports the single-change ablations on \textsc{Bunny} and \textsc{Statue} with all three metrics; Table~\ref{tab:ablmore} repeats their PSNR columns beside five further scenes, so the ordering can be read across every scene at once. Removing the shadow term is the largest single change on six of the seven scenes. \textsc{Bunny} is the exception and the clearest case for the channel: removing the subsurface term costs $5.48$\,dB there, the largest effect of any arm on any scene. What comes second follows the medium: the subsurface channel on \textsc{JoyfulYell}, \textsc{Translucent}, \textsc{CupFabric} and \textsc{Pikachu}, the transport cues on \textsc{Statue} and \textsc{Drums}. On \textsc{Drums}, whose translucent component is weak, removing the subsurface channel costs $0.10$\,dB, the least it costs on any scene, consistent with the degeneration of Remark~\ref{rem:degen}. The one design choice that does not move the rendered image, the specular leakage regularizer, is in Appendix~\ref{app:specreg}.

\section{The specular leakage regularizer}
\label{app:specreg}

\begin{figure}[!htb]
    \centering
    \begin{minipage}[t]{0.495\linewidth}
        \centering
        \includegraphics[width=\linewidth]{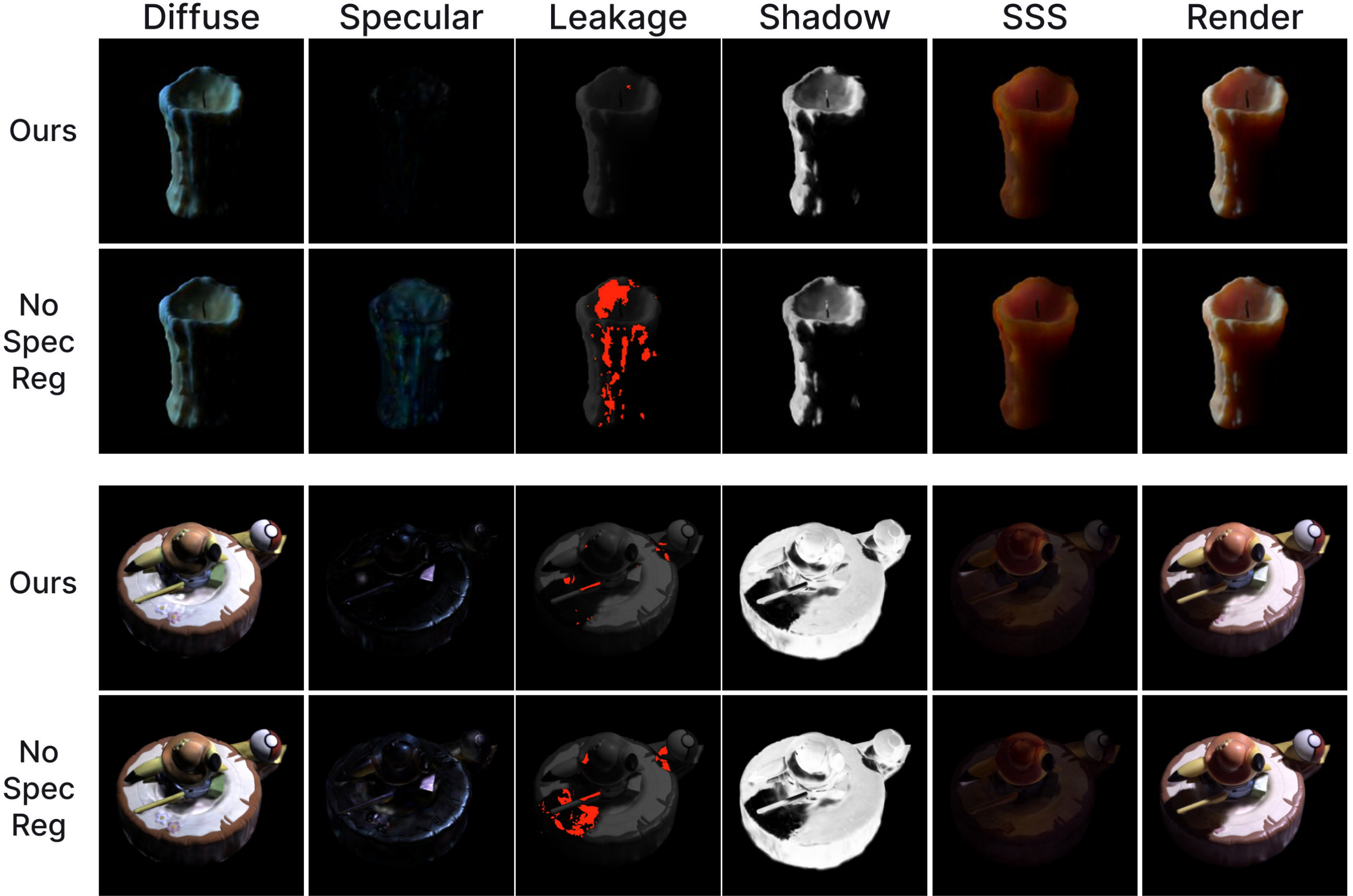}\\[3pt]
        {\small (a) Leakage removed by the regularizer}
    \end{minipage}\hfill
    \begin{minipage}[t]{0.495\linewidth}
        \centering
        \includegraphics[width=\linewidth]{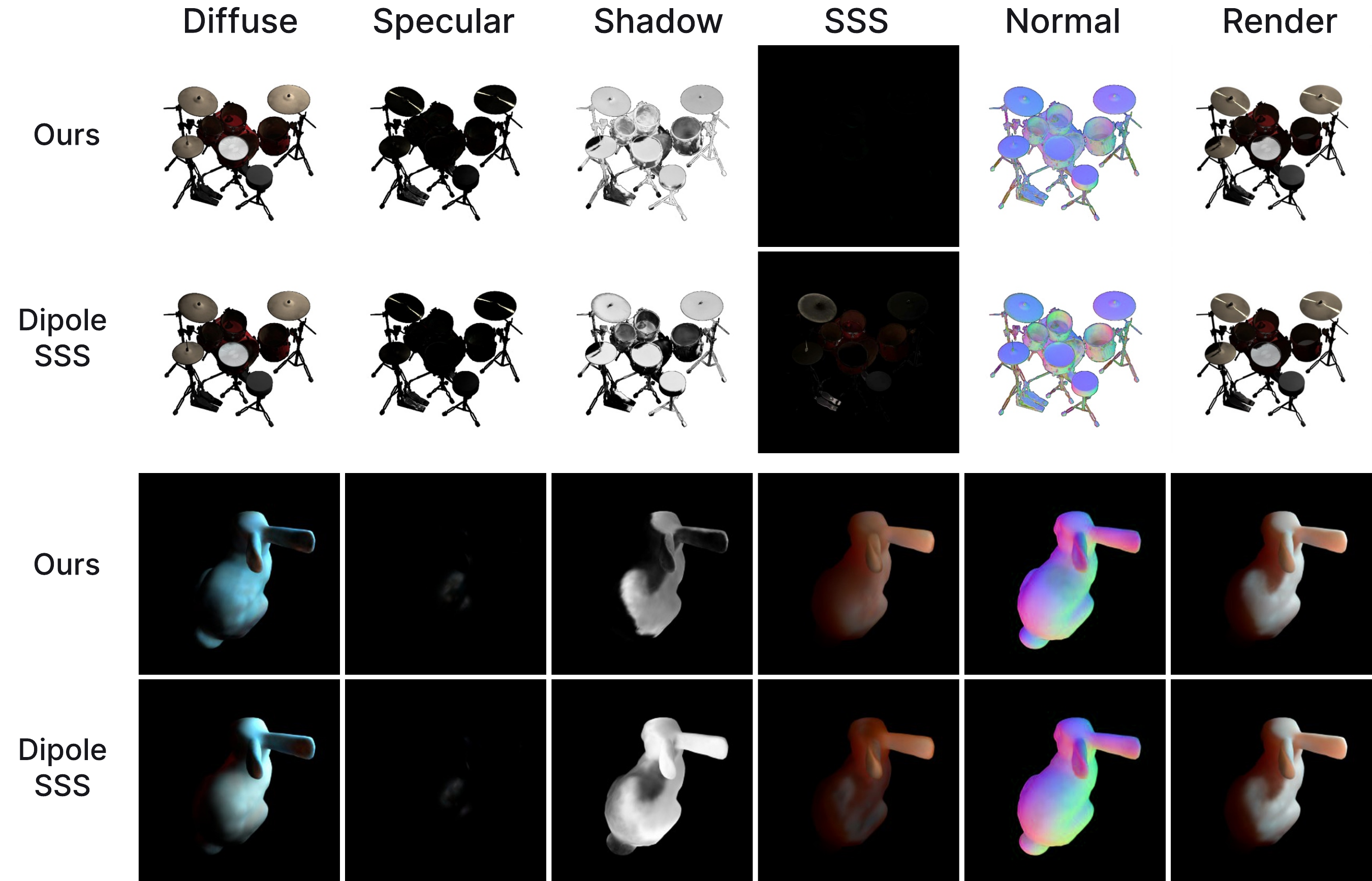}\\[3pt]
        {\small (b) Replacing the solve with the per-primitive kernel}
    \end{minipage}
    \caption{\textbf{Two reads of the subsurface slot.}
    \textbf{(a)} \textsc{Candle} and \textsc{Pikachu} with the specular leakage regularizer on
    (Ours) and off, same frame and channels. Without it the specular lobe paints a ghost of the
    object into shadowed regions and the diffuse term washes out. The \emph{Leakage} column marks
    in red the pixels where the model's own decay says the light is blocked ($D_i < 0.3$) while
    the specular lobe still fires above $0.05$.
    \textbf{(b)} \textsc{Drums} and \textsc{Bunny} with our solve (Ours) and with SSD-GS's
    per-primitive dipole in the same slot, everything else held fixed. \textsc{Drums} shows the channel taking up surface residual,
    \textsc{Bunny} the transported light misplaced. The SSS column of (b) is shown on its own black
    ground and at a $2.5\times$ gain so that both rows are legible.}
    \label{fig:attrib}
\end{figure}

Figure~\ref{fig:attrib}(a) shows the difference. This term does not earn PSNR and we do not
present it as if it did. Switching it on moves test PSNR by between $-0.30$ and $+0.14$\,dB across
the seven scenes where both arms exist, a mean of $-0.08$\,dB with a standard deviation of $0.14$,
and the sign changes from scene to scene.

The term earns its place by keeping the decomposition identifiable. Nothing in a photometric loss distinguishes a
warm glow that the specular lobe paints onto a shadowed pixel from the same glow carried there by
the medium, so the free lobe will absorb transported appearance whenever that is the cheaper fit.
The regularizer penalises specular energy in pixels the light does not reach, the
configuration in which a surface lobe cannot account for the radiance. On \textsc{Bunny} it
collapses the specular footprint from $68\%$ to $23\%$ of the object at unchanged PSNR, the same image attributed differently.

Switching the regularizer off changes the full model by $0.18$\,dB, within the same range as the
per-scene values above.

\section{Replacing the solve with a per-primitive kernel}
\label{app:dipole}

A diffusion profile is a kernel, not an outgoing radiance. The BSSRDF it belongs to integrates the
profile against the light arriving over the surface, $L_o(\mathbf{x}_o) = \int_A\!\int_\Omega
S(\mathbf{x}_i,\boldsymbol{\omega}_i,\mathbf{x}_o,\boldsymbol{\omega}_o)\,L_i\,\mathrm{d}\boldsymbol{\omega}\,\mathrm{d}A$,
so evaluating it requires knowing where light entered. The per-primitive kernel of SSD-GS does not
carry that integral. The profile is evaluated once, at a radius its network predicts, and scaled by
the summed inverse-square falloff of the lights. Our own coefficients are learned as well. The difference is that the dipole spreads the profile along
Euclidean distance from that radius, a route the medium need not occupy.

Thus we further investigate the effect of volumetric subsurface scattering against the dipole. Only
the subsurface slot changes, with the transport cues, the specular regularizer, the exit term and
the per-primitive amplitude held as in the full model. Table~\ref{tab:dipole} reports the
substitution, which costs accuracy on six of the seven scenes, by more than $2$\,dB on \textsc{CupFabric} and
\textsc{Pikachu}. Figure~\ref{fig:attrib}(b) shows where the
difference sits. On \textsc{Bunny} our channel is the smoother of the two and stays in the region
the decay leaves unlit, while the dipole spreads a comparable total over the whole body.
\textsc{Drums} is the one scene the substitution wins on PSNR and SSIM. Both channels pick up some surface residual
on an opaque object, the dipole by about two orders of magnitude more.

\begin{table}[!htb]\centering\footnotesize
\caption{\textbf{The dipole kernel in place of the transport solve.} Only the subsurface slot
changes: SSD-GS's own scattering module and its $\sum_\ell 1/d_\ell^2$ carrier replace our solve,
while the transport cues on the shadow branch, the specular regularizer, the exit term and the
per-primitive amplitude are left as in the full model.}
\label{tab:dipole}
\begin{tabular}{lccc ccc}
\toprule
 & \multicolumn{3}{c}{VolS-GS} & \multicolumn{3}{c}{Dipole SSS} \\
\cmidrule(lr){2-4}\cmidrule(lr){5-7}
 & PSNR$\uparrow$ & SSIM$\uparrow$ & LPIPS$\downarrow$ & PSNR$\uparrow$ & SSIM$\uparrow$ & LPIPS$\downarrow$ \\
\midrule
\textsc{Bunny}       & \textbf{40.15} & \textbf{0.9920} & \textbf{0.0112} & 39.52 & 0.9911 & 0.0119 \\
\textsc{Statue}      & \textbf{38.35} & \textbf{0.9899} & \textbf{0.0130} & 36.53 & 0.9862 & 0.0165 \\
\textsc{JoyfulYell}  & \textbf{32.19} & \textbf{0.9690} & \textbf{0.0385} & 31.68 & 0.9674 & 0.0408 \\
\textsc{Translucent} & \textbf{31.64} & \textbf{0.9783} & \textbf{0.0246} & 31.13 & 0.9778 & 0.0251 \\
\textsc{CupFabric}   & \textbf{37.02} & \textbf{0.9823} & \textbf{0.0516} & 34.61 & 0.9744 & 0.0550 \\
\textsc{Pikachu}     & \textbf{32.96} & \textbf{0.9689} & \textbf{0.0658} & 30.73 & 0.9592 & 0.0694 \\
\textsc{Drums}       & 33.21 & 0.9839 & \textbf{0.0155} & \textbf{33.24} & \textbf{0.9844} & 0.0158 \\
\bottomrule
\end{tabular}
\end{table}

\section{Decomposition across methods}
\label{app:compare}

\begin{figure}[!htb]
    \centering
    \includegraphics[width=\linewidth]{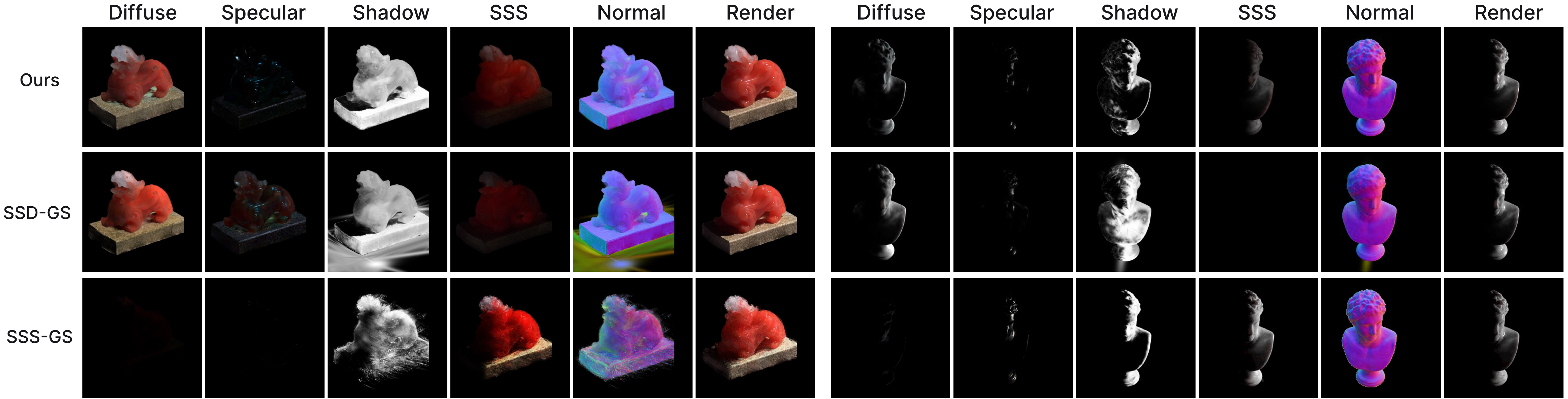}
    \caption{\textbf{Decomposition across methods.} \textsc{Pixiu} (left) and \textsc{Statue} (right),
    each row one method's own decomposition on the same frame. For SSS-GS the SSS column shows its residual
    network's output, its only channel with appreciable energy.}
    \label{fig:compare}
\end{figure}

Figure~\ref{fig:compare} reads the same decomposition across the two baselines that also
carry an explicit subsurface term, on \textsc{Pixiu} (left block) and
\textsc{Statue} (right block). On \textsc{Statue}, SSD-GS's subsurface channel is empty (channel-to-render
ratio $0.00$ against our $0.52$) while its rendered image still contains the backlit glow,
so the glow is not explained by its explicit subsurface component. SSS-GS's physical channels are nearly
dark on both scenes and its residual network carries more energy than its own final
render (ratio $1.5$ on \textsc{Statue}), indicating substantial energy that is not exposed as
the method's explicit subsurface component; its SSS column in the figure therefore shows that
residual. Our channel carries the backlit radiance itself, and
the render matches the reference.

SSD-GS's empty channel on \textsc{Statue} is the same failure that
Figure~\ref{fig:attrib}(b) isolates. The module whose subsurface channel reads $0.00$ here is the
per-primitive dipole we substituted into our own slot there, and in that controlled setting the
kernel either mis-places the transported light or fires where there is no medium to transport
through. A kernel with no reading of the medium has no stable quantity to put in
the channel, so the glow appears in the render but not in the decomposition.

\section{Failure cases}
\label{app:failure}

The thinnest geometry fails first. On \textsc{Bunny} the ear tip seen edge-on renders slightly dark, with less transmitted light than the capture (Figure~\ref{fig:failure}). The single-scattering term of Eq.~\ref{eq:single} exists for exactly this regime, where a lobe only a few mean free paths across is too thin for the diffusion approximation to be accurate, and it narrows the gap without closing it. Refining the transport grid from $64^3$ to $96^3$ moves the region only by $0.01$\,dB.

The second failure is on an opaque object. On \textsc{Hotdog} the plate under the bun renders darker than the capture. The appearance there is a broad and nearly uniform highlight-like glow, which our decomposition does not target.

\begin{figure}[!htb]
    \centering
    \includegraphics[width=\linewidth]{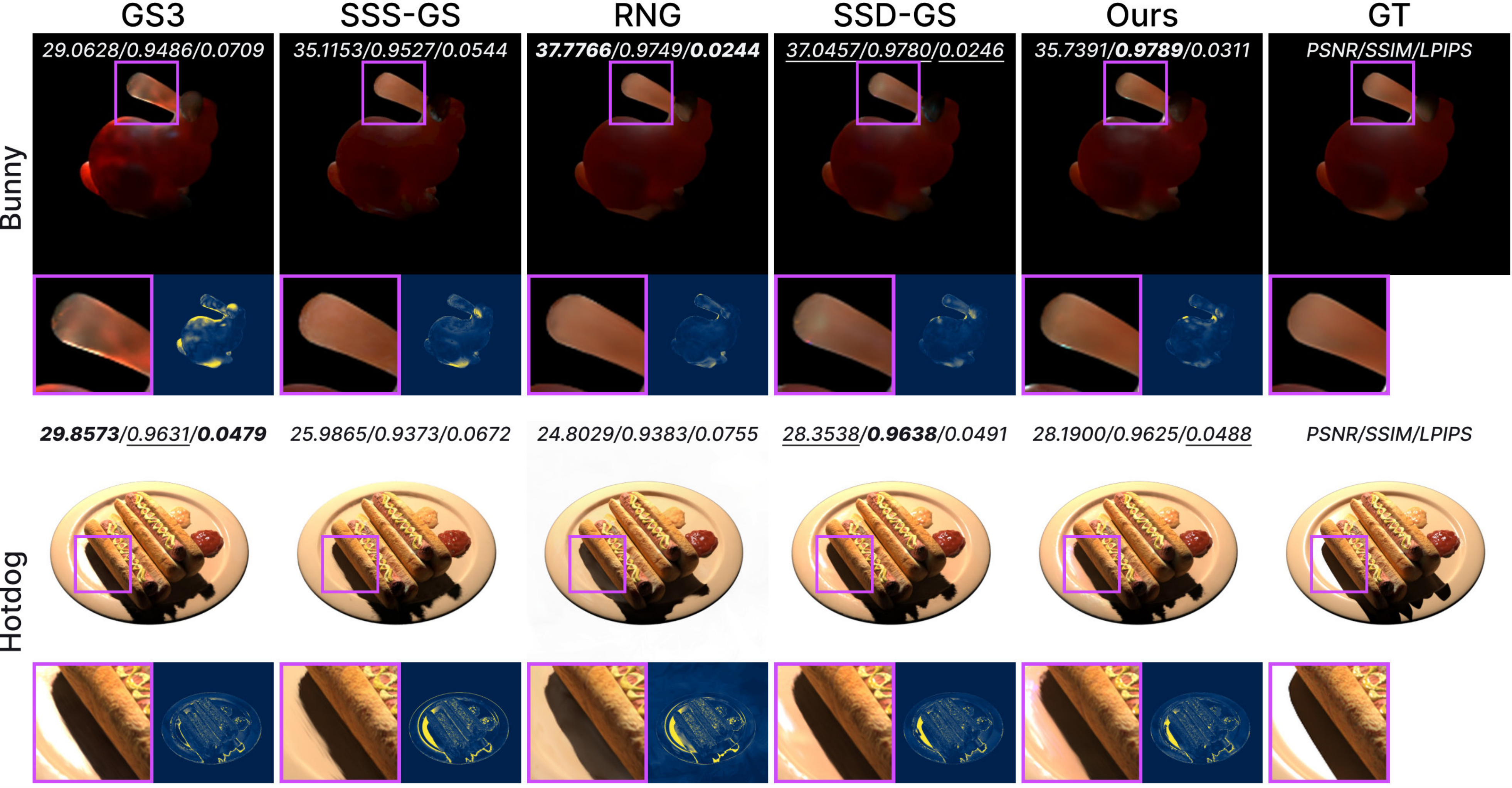}
    \caption{\textbf{Failure cases.} \textsc{Bunny} with the ear folded back and seen edge-on, where the tip renders darker than the capture, and \textsc{Hotdog}, where the plate under the bun renders darker. Under each render the magnified square marks the failing region at the same image coordinates in every column, beside the error map $\lvert\text{render}-\text{GT}\rvert$ averaged over color channels and clipped at $0.15$ on one shared ramp. Per-frame PSNR/SSIM/LPIPS are printed on each panel, \textbf{best} in bold and \underline{second} underlined within a row.}
    \label{fig:failure}
\end{figure}

\end{document}